\documentclass{new_tlp}
\usepackage{graphicx}
\usepackage{multirow}
\usepackage{xcolor}
\usepackage{xspace}
\usepackage{algorithm}
\usepackage{algorithmic}
\usepackage{amsmath}
\usepackage{mathtools}
\usepackage{xspace}
\usepackage{amssymb}
\usepackage{multirow}
\usepackage{todonotes}
\usepackage{subcaption}
\usepackage{tikz}
\usepackage{pgfplots}
\usepackage{url}
\pgfplotsset{compat=1.18}

\newcommand{\ASPQ}{ASP(Q)\xspace}
\newcommand{\ASPQW}{ASP\ensuremath{^w}(Q)\xspace}
\newcommand{\TWOASPQ}{2-ASP(Q)\xspace}
\newcommand{\TWOASPQW}{2-ASP\ensuremath{^w}(Q)\xspace}
\newtheorem{theorem}{Theorem}
\newtheorem{lemma}{Lemma}
\newtheorem{example}{Example}
\newtheorem{definition}{Definition}
\newtheorem{obs}{Observation}
\newtheorem{prop}{Proposition}
\newcommand{\fix}[2]{\mathit{fix_{#1}(#2)}}
\newcommand{\heads}[1]{\mathcal{H}(#1)}
\def\naf{\ensuremath{\raise.17ex\hbox{\ensuremath{\scriptstyle\mathtt{\sim}}}}\xspace}
\newcommand{\updated}[1]{{\color{black} #1}}
\definecolor{darkgreen}{RGB}{26, 72, 21}

  \title[\TWOASPQ with weak constraints]
        {2-\ASPQ programs with weak constraints: Complexity and efficient implementation\thanks{This work has been partially funded by the Italian Ministry of Industrial Development (MISE) under project EI-TWIN n. F/310168/05/X56 CUP B29J24000680005 and also partially by projects  SIGENERA (CUP J29I24001770005) and MOZART (J49I24001740005) selected within the framework of the PR FESR – FSE Calabria 2021/2027 and implemented with the support of the Italian State and the Calabria Region.}}

    \author[A. Cuteri, G. Mazzotta, F. Ricca]
         {ANDREA CUTERI \quad GIUSEPPE MAZZOTTA \quad FRANCESCO RICCA\\
         University of Calabria, Rende, Italy\\
         \email{\{andrea.cuteri,giuseppe.mazzotta,francesco.ricca\}@unical.it}}

\jdate{}
\pubyear{}
\pagerange{}
\doi{}

\begin{document}

\maketitle

    \begin{abstract}
        \ASPQ extends Answer Set Programming (ASP) with Quantifiers over answer sets. In this paper we focus on the class of \ASPQ programs with two quantifiers and weak constraints, denoted as \TWOASPQW. 
        \TWOASPQW is a practically relevant fragment of \ASPQ that is expressive enough to capture optimization problems up to the class $\Delta^P_3$.
        On the theoretical side, we provide a complete complexity characterization of the main computational tasks for \TWOASPQW programs, including tight completeness results and the analysis of nontrivial cases that have not been addressed in previous works.
        On the practical side, we introduce novel strategies for computing (optimal) quantified answer sets in the \textsc{casper} system, that rely on a Counterexample-Guided Abstraction Refinement (CEGAR) technique tailored to \ASPQ.
        An experimental evaluation on hard benchmarks from different application domains shows that the proposed techniques are effective in practice.
    \end{abstract}


\section{Introduction}
Answer Set Programming~\cite{DBLP:journals/cacm/BrewkaET11} (ASP) is a well-established declarative formalism that has been widely adopted for modeling and solving hard combinatorial problems. Over the years, ASP has been successfully applied in a variety of real-world domains, including planning~\cite{DBLP:journals/tplp/SonPBS23}, scheduling~\cite{DBLP:conf/aiia/DodaroGMP18,DBLP:conf/ruleml/CardelliniNDGGM21}, business process management~\cite{DBLP:conf/padl/ChiarielloFIR24,DBLP:conf/lpnmr/FiondaIR24}, among many others~\cite{DBLP:journals/tods/EiterFGL08}.
A substantial body of research has focused on extending ASP modeling and solving capabilities, leading to the proposal of several language extensions~\cite{DBLP:journals/tplp/AmendolaRT19,DBLP:journals/tplp/BogaertsJT16,DBLP:journals/tplp/FandinnoLRSS21}. 
Among these, Answer Set Programming with Quantifiers~\cite{DBLP:journals/tplp/AmendolaRT19} (\ASPQ) introduces the possibility to quantify over answer sets~\cite{DBLP:journals/ngc/GelfondL91}. 
This extension enhances the expressive power of ASP, allowing for the natural modeling of problems spanning the entire Polynomial Hierarchy (PH).
\ASPQ has found interesting applications in several contexts, including outlier detection~\cite{DBLP:conf/padl/BellusciMR22}, planning~\cite{DBLP:conf/padl/FaberMC22} and related domains~\cite{DBLP:conf/iclp/000124}, establishing it as a promising framework for solving problems beyond NP.
More recently, \ASPQ has been extended with \textit{weak constraints}, denoted by \ASPQW~\cite{DBLP:journals/tplp/MazzottaRT24}, further enhancing its ability to naturally represent hard optimization problems.
Weak constraints in \ASPQW can be applied both locally to enable quantification over optimal answer sets, and globally to express preferences among solutions, thus capturing optimization problems throughout the PH.
While weak constraints improve the modeling capabilities of \ASPQ, they introduce additional sources of computational complexity, making the design of effective solving techniques challenging~\cite{DBLP:conf/padl/AzzoliniLMR25}.

Within this context, we focus on \ASPQ programs with two quantifiers and weak constraints, denoted as \TWOASPQW, which are expressive enough to model complex optimization problems up to $\Delta^P_3$. Indeed, the expressive power of \TWOASPQW is not yet fully characterized, as existing completeness results are available only for specific language fragments~\cite{DBLP:journals/tplp/MazzottaRT24}.
As a consequence, the development of efficient and complexity-aware implementations for the entire \TWOASPQW has so far been left unexplored.
To fill this gap, this paper undertakes a complexity analysis of the main reasoning tasks for \TWOASPQW programs, namely coherence and brave reasoning, and derives tight completeness results for the general case.

Moreover, the paper introduces efficient solving techniques for \TWOASPQW.
In particular, we propose an efficient evaluation technique based in Counterexample-Guided Abstraction Refinement (CEGAR)~\cite{DBLP:journals/ijfcs/ClarkeFHKOST03,DBLP:journals/ai/JanotaKMC16}, which serves as the foundation for two algorithms computing optimal quantified answer sets. 
Although inspired by the lower and upper-bound improving strategies used in ASP solvers~\cite{DBLP:journals/logcom/AlvianoD0R20}, our lower-bound improving approach departs from existing methods by realizing lower-bound improvement via abstraction refinement instead of unsatisfiable-core extraction, thereby yielding a distinctive and novel technique.

Finally, the proposed techniques have been implemented on top of the \textsc{casper} system~\cite{Cuteri_Mazzotta_Ricca_2026}. Experimental results demonstrate the effectiveness of our approach across several benchmarks drawn from diverse application domains, confirming both its practical viability and its alignment with the theoretical complexity results.

\section{Background}
In this section, we provide some preliminaries and notation on ASP and \ASPQ.

\subsection{Answer Set Programming}

\paragraph{Syntax}
A term is a constant (i.e., an integer or a string starting with lowercase letter) or a variable (i.e., a string starting with uppercase letter).
An \textit{atom} is an expression of the form $p(\vec{t})$ with $\vec{t}=t_1,\ldots,t_n$ being a list of terms and $p$ being a \textit{predicate} of arity $n\ge0$. 
An atom $p(\vec{t})$ is \textit{ground} if all terms in $\vec{t}$ are constants. 
A \textit{literal} is either an atom $a$ or its negation $\naf a$.
The \textit{complement} of a literal $l = a$ (resp. $l=\naf a$) is $\overline{l} = \naf a$ (resp. $\overline{l} = a$). 
Given a set of literals $L$, $L^+$ (resp. $L^-$) denotes the set of positive (resp. negative) literals appearing in $L$.
A \textit{rule} is an expression of the form
$h \leftarrow l_1, \ldots, l_n$
where $h$ is an atom referred to as \textit{head}, and $l_1,\ldots,l_n$ with $n\ge0$ is a conjunction of literals referred to as \textit{body}. A rule with empty body is called \textit{fact}, while a rule with empty head is called \textit{hard constraint}.
A \textit{weak constraint} is an expression of the form
$\leftarrow^{\omega} l_1, \ldots, l_n\ [w@l,\vec{t}]$
where $l_1, \ldots, l_n$ is a conjunction of literals referred to as body, $w$ and $l$ are terms, and $\vec{t}= t_1, \ldots, t_m$ is a possibly empty list of terms.
Given a rule (resp. a weak constraint) $r$, $H_r$ denotes the set of atoms appearing in the head of $r$ while $B_r$ denotes the set of literals appearing in the body of $r$.
A rule $r$ (resp. a weak constraint) is \textit{safe} if each variable appears in at least one literal in $B_r^+$.
A \textit{program} $P$ is a set of safe rules and weak constraints.
Given a program $P$, $\mathcal{R}(P)$ and $\mathcal{W}(P)$ denote, respectively, the sets of rules and weak constraints appearing in $P$; while $\heads{P}$ denotes the set of atoms appearing in the head of some rules in $P$.
Given an expression $\epsilon$ (atom, program, etc.), $at(\epsilon)$ denotes the set of atoms appearing in $\epsilon$.

\paragraph{Semantics}
Given an ASP program $P$, the \textit{Herbrand Universe}, $HU_P$, of $P$ is the set of all constants appearing in $P$; the \textit{Herbrand Base}, $B_P$, is the set of all possible ground atoms that can be obtained from predicates and constants in $P$; $ground(P)$ denotes the set of all ground rules that can be obtained from $P$ by proper substitutions of variables in $P$ with constants in $HU_P$.
An \textit{interpretation} $I \subseteq B_P$ is a set of atoms. A ground literal $l=a$ (resp. $l=\naf a$) is true w.r.t. $I$ if $a \in I$ (resp. $a \notin I$), false otherwise.
A conjunction of literals $conj$ is true w.r.t. $I$ if all the literals in $conj$ are true w.r.t. $I$, false otherwise.
An interpretation $I$ is an \textit{answer set} of $P$ iff ($i$) $I$ is a model of $P$, namely for each rule $r \in ground(P)$ either $H_r$ is true w.r.t. $I$ or $B_r$ is false w.r.t. $I$; and ($ii$) $I$ is a minimal model of its GL-reduct~\cite{DBLP:journals/ngc/GelfondL91}.
Let $AS(P)$ be the set of answer set of a program $P$, then $P$ is \textit{coherent} iff $AS(P)\neq \emptyset$. 
For a program $P$ and an interpretation $I$, let the set of weak constraint violations be 
$ws(P,I) = \{(w,l,\vec{t}) \mid \ \leftarrow^{\omega} b_1,\ldots,b_m\ [w@l,\vec{t}] \in ground(P)$ and $b_1,\ldots,b_m$ are true w.r.t. $I\}$, then the cost function of $P$ is defined as
$\mathcal{C}(P,I,k) = \sum_{(w,k,\vec{t}) \in ws(P,I)} w$ for every integer $k$.
Let $M_1,M_2 \in AS(P)$ then $M_1$ is \textit{dominated} by $M_2$ if there exists an integer $l$ such that $\mathcal{C}(P,M_1,l)>\mathcal{C}(P,M_2,l)$ and for each $l'>l$, $\mathcal{C}(P,M_1,l') = \mathcal{C}(P,M_2,l')$.
Let $M \in AS(P)$ then $M$ is an \textit{optimal answer set} iff $M$ is not dominated by any $M' \in AS(P)$. We denote by $OptAS(P)$ the set of optimal answer set of $P$.

\subsection{ASP with Two Quantifiers}
A \TWOASPQW program~\cite{DBLP:journals/tplp/MazzottaRT24} is an expression of the form $\Box_1 P_1 \Box_2 P_2: C : C^\omega$, where $\Box_1,\Box_2$ are quantifiers in $\{\exists^{st}, \forall^{st}\}$, $P_1,P_2$ are ASP programs possibly with weak constraints, $C$ is a stratified program~\cite{DBLP:books/sp/CeriGT90} with hard constraints, and $C^\omega$ is a set of weak constraints such that $B_{C^{\omega}}\subseteq B_{P_1}$.  
Weak constraints appearing $P_1$ and $P_2$ are said \textit{local}; whereas those in $C^{\omega}$ are said \textit{global}.
A \TWOASPQW program $\Pi$ is said to be \textit{existential} if $\Box_1 = \exists^{st}$, otherwise it is \textit{universal}.
Moreover, $\Pi$ is said to be \textit{alternating} if $\Box_1\neq \Box_2$, \updated{and} \textit{plain} if it contains no weak constraints.
We now define the semantics of \TWOASPQW programs.
Let $P$ be an ASP program and $M \subseteq B_P$, then $\fix{P}{M}$ denotes the set of facts and hard constraints of the form $\{a\leftarrow\mid a \in M\} \cup \{\leftarrow a\mid a \in B_P\setminus M\}$.
Then, $M \in AS(P)$ \textit{satisfies} a program $P'$ if $P'\cup\fix{P}{M}$ is coherent.

\updated{At this point the coherence of \TWOASPQW can be defined as follows:}
\begin{itemize}
    \item $\exists^{st} P: C :C^{\omega}$ is coherent iff there exists $M \in OptAS(P)$ such that $M$ satisfies $C$.
    \item $\forall^{st} P: C :C^{\omega}$ is coherent iff for each $M \in OptAS(P)$, $M$ satisfies $C$.
    \item $\exists^{st} P_1 \Box_2 P_2: C :C^{\omega}$ is coherent iff there exists $M_1 \in OptAS(P_1)$ such that $\Box_2 P_2 \cup\fix{P_1}{M_1}: C :C^{\omega}$ is coherent.
    \item $\forall^{st} P_1 \Box_2 P_2: C :C^{\omega}$ is coherent iff for each $M_1 \in OptAS(P_1)$, $\Box_2 P_2 \cup\fix{P_1}{M_1}: C :C^{\omega}$ is coherent.
\end{itemize}
For an existential \TWOASPQW program $\Pi$, an optimal answer set $M_1 \in OptAS(P_1)$ is a \textit{quantified answer set} of $\Pi$ if $\Box_2 P_2\cup\fix{P_1}{M_1}:C:C^{\omega}$ is coherent. 
We denote by $QAS(\Pi)$ the set of quantified answer sets of $\Pi$.
Let $\Pi$ be an existential \TWOASPQW program, $l$ be an integer, and $M \in QAS(\Pi)$, then the cost of $M$ at level $l$ is defined as $\mathcal{C}(M,\Pi,l) = \mathcal{C}(M,P_1\cup C^{\omega},l)$.
Let $M_1,M_2 \in QAS(\Pi)$, then $M_1$ is \textit{dominated} by $M_2$ if there exists $l$ such that $\mathcal{C}(M_1,\Pi,l) > \mathcal{C}(M_2,\Pi,l)$ and for each $l'>l$, $\mathcal{C}(M_1,\Pi,l') = \mathcal{C}(M_2,\Pi,l')$. 
Thus, $M$ is an \textit{optimal quantified answer set} if $M$ is not dominated by any $M'\in QAS(\Pi)$. We denote by $OptQAS(\Pi)$ the set of optimal quantified answer sets.

\updated{
\begin{example}
    Let $\Pi$ be a \TWOASPQW program of the form $\exists^{st} P_1\forall^{st}P_2:C$, where $C=\{\leftarrow nb,\ nc\}$ and
    \begin{minipage}{.49\textwidth}
        \[
            P_1 = \left\{\begin{array}{lr}
                a  \leftarrow \naf na\qquad      & b  \leftarrow \naf nb\\
                na \leftarrow \naf a\qquad       & nb \leftarrow \naf b\\
              \end{array}\right\}
        \]
    \end{minipage}
    \begin{minipage}{.49\textwidth}
        \[
            P_2 = \left\{\begin{array}{ll}
                c  \leftarrow \naf nc\qquad & \leftarrow^{\omega} a,\ \naf c\ [1@1]\\
                nc \leftarrow \naf c\qquad  & \leftarrow^{\omega} b,\ \naf nc\ [1@1]\\
              \end{array}\right\}
        \]    
    \end{minipage}\smallskip

    \noindent Let us consider $M_1 = \{na,nb\} \in AS(P_1)$. In this case, $\{na,nb,nc\} \in OptAS(P_2\cup\fix{P_1}{M_1})$ violates the constraint $\leftarrow nb,nc \in C$. Thus, $M_1\notin QAS(\Pi)$.
    Conversely, $M_1' = \{a,nb\}\in AS(P_1)$ is such that $P_2\cup\fix{P_1}{M_1'}$ admits only one optimal answer set $M_2 = \{a,nb,c\}$ which satisfies the constraint in $C$ and so $M_1\in QAS(\Pi)$.
    If we remove weak constraints from $P_2$, then $AS(P_2\cup\fix{P_1}{M_1}) = \{M_2, M_2'\}$ where $M_2' = \{a,nb,nc\}$. Here, $M_2'$ violates the constraint in $C$ and so $M_1$ is not a quantified answer set anymore.
\end{example}
}

\section{Complexity Results for \TWOASPQW}
In this section, we study the main computational tasks for \TWOASPQW: Coherence and Brave reasoning.
For \TWOASPQW programs, the coherence problem checks whether a program is coherent, while brave reasoning verifies whether an atom occurs in an optimal quantified answer set.
Unlike standard ASP, local weak constraints in \ASPQW programs introduce an additional source of computational complexity, as observed by~\citeN{DBLP:journals/tplp/MazzottaRT24}. 
Even though complexity results for these tasks have been established for specific subclasses of \ASPQW, a complete characterization is still missing. 

In particular, it was shown that verifying the coherence for \TWOASPQW programs ($i$) is in $\Sigma_3^P$ for existential programs and in $\Pi_3^P$ for universal programs, and ($ii$) is hard for $\Sigma_2^P$ and $\Pi_2^P$, respectively. 
Starting from these results, we prove that the coherence problem for \TWOASPQW is complete for the second level of the PH, namely $\Sigma_2^P$ for existential programs and $\Pi_2^P$ for universal programs
(full proofs in~\ref{sec:complexity_res}).
\begin{theorem}[Membership]\label{thm:membership}
    The coherence problem for \TWOASPQW programs of the form $\Box_1 P_1\Box_2 P_2:C$ is in: $\Sigma_2^P$ if $\Box_1 = \exists^{st}$; $\Pi_2^P$ otherwise,
    \updated{
    no matter whether $\Box_2 = \exists^{st}$ or $\Box_2 = \forall^{st}$.
    }
\end{theorem}

\begin{proof}[Proof (Sketch)]
    To prove our thesis we need to distinguish three different scenarios.
    
    Uniform quantifiers $\Box_1 = \Box_2$. By applying the transformation (Algorithm 1) by \citeN{DBLP:journals/tplp/MazzottaRT24}, it is possible to obtain a plain 2-\ASPQ program $\Pi'$ such that $\Pi'$ is coherent if and only if $\Pi$ is coherent. 
    From the result of~\citeN{DBLP:journals/tplp/AmendolaRT19}, verifying the coherence of $\Pi'$ is $\Sigma_2^P$-complete if $\Box_1 = \exists^{st}$; otherwise is in $\Pi_2^P$-complete. 
    Thus the thesis holds for uniform quantifiers.

    Case $\Box_1 = \exists^{st}$ and $\Box_2 = \forall^{st}$. 
    In this case, it is possible to translate $\Pi$ into an existential \TWOASPQW program $\Pi'$ where weak constraints appear only in the first subprogram. Since verifying the coherence of $\Pi'$ is $\Sigma_2^P$-complete \updated{(Corollary 3.1 by~\citeN{DBLP:journals/tplp/MazzottaRT24})}, then the thesis holds also in this case.
    
	More precisely, the \TWOASPQW program $\Pi'$ can be obtained by copying the rules of the program $P_2$ in $P_1$ and then adding rules in the program $C$ for verifying the optimality of an answer sets of $P_2$.
	As a result, we obtain $\Pi'$ of the form $\exists P_1' \forall P_2': C'$ where:
	\begin{itemize}
		\item optimal answer sets $M$ of $P_1'$ are of the form $M_1 \cup M_2^c$ where $M_1$ is an optimal answer set of $P_1$ and $M$ corresponds to an answer set of $P_2\cup\fix{P_1}{M_1}$;
		\item answer sets $N$ of $P_2'\cup\fix{P_1'}{M}$ are of the form $M_1 \cup M_2^c \cup M_2$ where $M_1\cup M_2$ is an answer set of $P_2\cup\fix{P_1}{M_1}$;
		\item an answer set $N$ of $P_2'\cup\fix{P_1'}{M}$ satisfies $C'$ if and only if $M_1 \cup M_2^c$ is not dominated by $M_1 \cup M_2$ and one of the following holds:
		\begin{enumerate}
			\item there exists $l$ such that the cost of $M_1\cup M_2$ is different from the cost of $M_1\cup M_2^c$;\label{cond1:sat_c_member}
			\item the cost of $M_1\cup M_2$ is equal to the cost of $M_1\cup M_2^c$ for each level and $M_1 \cup M_2$ satisfies $C$.\label{cond2:sat_c_member}
		\end{enumerate}
	\end{itemize}
	Let $M = M_1 \cup M_2^c$ be an optimal answer set of $P_1'$.
	If $M$ is not an optimal answer set of $P_2\cup \fix{P_1}{M_1}$, then $M$ cannot be a quantified answer set as there exists $N = M_1 \cup M_2^c \cup M_2$ in the answer sets of $P_2'\cup\fix{P}{M}$, where $M_1 \cup M_2$ is an optimal answer set of $P_2\cup\fix{P_1}{M_1}$ and $N$ violates $C'$ since $M_1\cup M_2^c$ is dominated by $M_1 \cup M_2$.
	On the other hand, since $M$ is an optimal answer set of $P_2\cup\fix{P_1}{M_1}$, then for each $N = M_1 \cup M_2^c \cup M_2$ in the answer sets of $P_2'\cup\fix{P_1'}{M}$, $M_1\cup M_2^c$ is not dominated by $M_1\cup M_2$ (i.e., Condition~\ref{cond1:sat_c_member} is always satisfied). 
    Thus, let $N = M_1\cup M_2^c \cup M_2$ be an answer set of $P_2\cup \fix{P_1'}{M}$, then $N$ satisfies $C'$ if and only if one between conditions~\ref{cond1:sat_c_member} and~\ref{cond2:sat_c_member} holds. 
    Here we can observe that  Condition~\ref{cond1:sat_c_member} holds if and only if $M_1\cup M_2$ is not an optimal answer set of $P_2\cup\fix{P_1}{M_1}$; whereas Condition~\ref{cond2:sat_c_member} holds if and only if $M_1\cup M_2$ is an optimal answer set of $P_2\cup\fix{P_1}{M_1}$ and $M_1\cup M_2$ satisfies $C$.
	Thus, $M$ is a quantified answer set of $\Pi'$ if and only if each optimal answer set $M_1\cup M_2$ of $P_2\cup\fix{P_1}{M_1}$ satisfies $C$. Hence, $M$ is a quantified answer set of $\Pi'$ if and only if $M_1$ is a quantified answer set of $\Pi$. 
    Finally, we can conclude that $\Pi$ is coherent iff $\Pi'$ is coherent. 
    
    Case $\Box_1 = \forall^{st}$ and $\Box_2 = \exists^{st}$. 
    By following the same working principle as before, it is possible to encode $\Pi$ into a \TWOASPQW program $\Pi'$ which preserves the coherence of $\Pi$ and in which weak constraints appear only in the first subprogram. In this case, verifying the coherence of $\Pi'$ is $\Pi_2^P$-complete \updated{(Corollary 3.1 by~\citeN{DBLP:journals/tplp/MazzottaRT24})}, then the thesis holds also in this last case.
    In this case, the \TWOASPQW program $\Pi'$ is of the form $\forall P_1' \exists P_2': C'$ where:
    \begin{itemize}
    	\item optimal answer sets $M$ of $P_1'$ are of the form $M_1 \cup M_2^c$ where $M_1$ is an answer set of $P_1$ and $M$ is an answer set of $P_2\cup\fix{P_1}{M_1}$;
    	\item answer sets $N$ of $P_2'\cup\fix{P_1'}{M}$ are of the form $M_1 \cup M_2^c \cup M_2$ where $M_1\cup M_2$ is an answer set of $P_2\cup\fix{P_1}{M_1}$;
    	\item an answer set $N$ of $P_2'\cup\fix{P_1'}{M}$ satisfies $C'$ if and only if one of the following holds:
    	\begin{enumerate}
    		\item $M_1 \cup M_2^c$ is dominated by $M_1 \cup M_2$;\label{cond1:u:sat_c_member}
    		\item the cost of $M_1\cup M_2$ is equal to the cost of $M_1\cup M_2^c$ for each level and $M_1 \cup M_2$ satisfies $C$.~\label{cond2:u:sat_c_member}
    	\end{enumerate}
    \end{itemize}
    According to the \ASPQW semantics, $\Pi'$ is incoherent if and only if there exists an optimal answer set $M_1$ of $P_1'$ such that for every answer set $N$ of $P_2'\cup\fix{P_1'}{M})$, $N$ does not satisfy $C'$.
    Let $M = M_1 \cup M_2^c$ be an optimal answer set of $P_1'$, with $M_1\in OptAS(P_1)$ and $M \in AS(P_2\cup\fix{P_1}{M_1})$.
    
    If $M_1 \cup M_2^c$ is not an optimal answer set of $P_2\cup\fix{P_1}{M_1}$, then there exists
    $M_1 \cup M_2 \in OptAS(P_2\cup\fix{P_1}{M_1})$ such that $M_1 \cup M_2^c$ is dominated by $M_1 \cup M_2$. Consequently,
    $N = M_1 \cup M_2^c \cup M_2 \in AS(P_2'\cup\fix{P}{M})$ is such that 
    Condition~\ref{cond1:u:sat_c_member} is satisfied, and thus $M$ cannot witness the incoherence of $\Pi'$.
    
    Otherwise, $M_1 \cup M_2^c$ is optimal for $P_2\cup\fix{P_1}{M_1}$. 
    In this case, for every
    $N = M_1 \cup M_2^c \cup M_2 \in AS(P_2'\cup\fix{P_1'}{M})$, $N$ satisfies $C'$ if and only if
    Condition~\ref{cond2:u:sat_c_member} holds, that is, if and only if
    $M_1 \cup M_2 $ has the same cost of $M_1\cup M_2^c $ (i.e., $M_1\cup M_2\in OptAS(P_2\cup\fix{P_1}{M_1})$) and $M_1 \cup M_2$ satisfies $C$.
    
    Therefore, $\Pi'$ is incoherent if and only if there exists $M_1 \in AS(P_1)$ such that no $M_1 \cup M_2 \in OptAS(P_2\cup\fix{P_1}{M_1})$ satisfies $C$, which is exactly the condition for $\Pi$ to be incoherent. 
    Hence, $\Pi'$ preserves the coherence of $\Pi$. 
\end{proof}

Note that, the presence of weak constraints in both quantified programs do not allow for a simple quantifier elimination, and the existing translation proposed by~\citeN{DBLP:journals/tplp/MazzottaRT24} requires the introduction of an additional quantifier, resulting in a non necessary jump in complexity to the third level of the PH in case $\Box_1 \neq \Box_2$. 

\begin{theorem}[Hardness]\label{thm:hardness}
    The coherence problem for \TWOASPQW programs of the form $\Box_1 P_1\Box_2 P_2:C$ is hard for: 
    $\Sigma_2^P$ if $\Box_1 = \exists^{st}$; $\Pi_2^P$ otherwise,
    \updated{
    no matter whether $\Box_2 = \exists^{st}$ or $\Box_2 = \forall^{st}$.
    }
\end{theorem}
\begin{proof}[Proof (Sketch)]
    Let us consider the different combination of quantifiers separately.
    
    \noindent (Case $\Box_1 = \Box_2 = \exists^{st}$). \updated{From Thm 4 by~\citeN{DBLP:journals/tplp/MazzottaRT24}, deciding the coherence of $\Pi$} is $\Sigma_2^P$-complete if $\mathcal{W}(P_1) = \emptyset$. Thus, the thesis holds since this is a particular case.
    
    \noindent (Case $\Box_1 = \Box_2 = \forall^{st}$).
    Let $\Phi = \forall X\exists Y \phi$ be a $2$-QBF, where $X$ and $Y$ are disjoint set of variables and $\phi$ is a boolean formula in 3-CNF. 
    Verifying that $\Phi$ is true is a $\Pi_2^P$-complete problem~\cite{schaefer2002completeness}. 
    We encode in polynomial time any $\Phi$ in a \TWOASPQW program of the form $\forall^{st} P_1 \forall^{st} P_2:C$, as detailed in~\ref{sec:complexity_res}.
    
    \noindent (Case $\Box_1 \neq \Box_2$). 
    Deciding the coherence of plain alternating \TWOASPQ programs $\Sigma_2^P$-complete if $\Box_1 = \exists^{st}$; otherwise it is $\Pi_2^P$-complete~\cite{DBLP:journals/tplp/AmendolaRT19}. Since this is a particular case of \TWOASPQW (i.e. $\mathcal{W}(P_1) = \mathcal{W}(P_2) = \emptyset$), then the thesis follows. 
\end{proof}
\begin{theorem}[Completeness]\label{thm:complete_two_aspq}
    The coherence problem for \TWOASPQW programs of the form $\Box_1 P_1\Box_2 P_2:C$ is: $\Sigma_2^P$-complete if $\Box_1 = \exists^{st}$; $\Pi_2^P$-complete otherwise,
    \updated{no matter whether $\Box_2 = \exists^{st}$ or $\Box_2 = \forall^{st}$.}
\end{theorem}
\begin{proof}
    The thesis follows from Theorem~\ref{thm:membership} and~\ref{thm:hardness}.
\end{proof}

From Theorem~\ref{thm:complete_two_aspq}, we derive complete results for the brave reasoning task in \TWOASPQW, which asks whether an atom $a$ is true in some optimal quantified answer set.
\begin{theorem}
    Let $\Pi$ be an existential \TWOASPQW program of the form $\Box_1 P_1\Box_2 P_2:C:C^{\omega}$ and $a$ be a ground atom, then verifying whether $a \in M$, with $M \in OptQAS(\Pi)$, is $\Delta_{3}^P$-complete \updated{
    no matter whether $\Box_2 = \exists^{st}$ or $\Box_2 = \forall^{st}$.
    }
\end{theorem}
\updated{These results strengthen those of~\citeN{DBLP:journals/tplp/MazzottaRT24}, as they establish completeness for all combinations of quantifiers, rather than for alternating programs only.}

\section{Solving \TWOASPQW via CEGAR}
Counterexample-Guided Abstraction Refinement (CEGAR)~\cite{DBLP:journals/ijfcs/ClarkeFHKOST03} is an iterative technique that starts from a simplified system representation \updated{(abstraction)} and progressively refines it using counterexamples, until a solution is found or non-existence is proven.
CEGAR-based techniques were successfully applied to QBF-satisfiability~\cite{DBLP:journals/ai/JanotaKMC16} and \TWOASPQ solving~\cite{Cuteri_Mazzotta_Ricca_2026}.
Building on these ideas, we propose a CEGAR-based approach for the evaluation of \TWOASPQW programs. In the following we give a game-theoretic characterization of the semantics of \TWOASPQW and then we focus on the two stages of CEGAR: \emph{abstraction refinement} and \emph{counterexample search}. 

\subsection{The \TWOASPQW Game}\label{sec:game_based_2aspqw}
Let $\Box \in \{\exists^{st},\forall^{st}\}$ be a quantifier, then the \emph{opponent} of $\Box$ is $\overline{\Box} = \forall^{st}$ if $\Box = \exists^{st}$; otherwise $\overline{\Box} = \exists^{st}$. 
Thus, for any alternating \TWOASPQW program $\Box_2$ is the opponent of $\Box_1$, that is, $\Box_2 = \overline{\Box_1}$. 
In what follows, w.l.o.g., we consider  \TWOASPQW programs of the form:
\begin{equation}\label{eq:2_aspqw}
    \Box P_1\overline{\Box} P_2:C:C^{\omega}
\end{equation}   
Let $\Pi$ be a \TWOASPQW program of the form~(\ref{eq:2_aspqw}), then a \emph{move} for $\Box$ is an answer set $M_1 \in OptAS(P_1)$. 
Given a move $M_1$ for $\Box$, for each $M_2 \in OptAS(P_2 \cup \fix{P_1}{M_1})$, $M_2|_{\heads{P_2}}$ is a \emph{candidate countermove} to $M_1$ for $\overline{\Box}$. 
Let $M_2 \in OptAS(P_2\cup\fix{P_1}{M_1})$ be a candidate countermove then $M_2|_{\heads{P_2}}$ is effectively a \emph{countermove} to $M_1$ for $\overline{\Box}$ if: ($i$) $\Box=\exists^{st}$ and $M_2$ does not satisfies $C$; or ($ii$) $\Box=\forall^{st}$ and $M_2$ satisfies $C$.
A \emph{winning move} for $\Box$ is a move $M_1$ such that no countermove to $M_1$ for $\overline{\Box}$ exists.
Thus, $\Box = \exists^{st}$ wins if there exists $M_1 \in OptAS(P_1)$ such that no countermove to $M_1$ for $\overline{\Box}$ exists.
\begin{prop}\label{prop:winning_move}
    Let $\Pi$ be a \TWOASPQW of the form~(\ref{eq:2_aspqw}), there exists a winning move for $\Box$ if and only if ($i$) $\Box=\exists^{st}$ and $\Pi$ is coherent, or ($ii$) $\Box = \forall^{st}$ and $\Pi$ is incoherent. 
\end{prop}
Intuitively, the notion of winning moves closely corresponds to the definition of coherence of \TWOASPQW programs. Let $\Pi$ be a \TWOASPQW program of the form~(\ref{eq:2_aspqw}), and let $M_1$ be a winning move for $\Box$.
If $\Box = \exists^{st}$, then there is no $M_2 \in OptAS(P_2 \cup \fix{P_1}{M_1})$ such that $M_2$ violates $C$. This implies that $\forall^{st} P_2 \cup \fix{P_1}{M_1} : C$ is coherent, and hence $\Pi$ is coherent.
Conversely, if $\Box = \forall^{st}$, then there is no $M_2 \in OptAS(P_2 \cup \fix{P_1}{M_1})$ such that $M_2$ satisfies $C$. Therefore, $\exists^{st} P_2 \cup \fix{P_1}{M_1} : C$ is incoherent, and thus $\Pi$ is incoherent.

\subsection{Counterexample search in \TWOASPQW}\label{sec:cegar_for_2aspqw}
Given a \TWOASPQW program $\Pi$ of the form (\ref{eq:2_aspqw}), from Proposition~\ref{prop:winning_move}, the coherence of $\Pi$ can be decided by searching for a winning move for $\Box$. 
Thus, the program $P_1$, whose optimal answer sets coincide with the possible moves of $\Box$, is a natural abstraction for $\Pi$. 
Hence, given $M_1\in OptAS(P_1)$, the counterexample search aims at contradicting $M_1$, which means finding a countermove to $M_1$ for $\overline{\Box}$. 
To this end, we define the \emph{countermove program} whose optimal answer sets correspond to countermoves to $M_1$ for $\overline{\Box}$. 

We recall that countermoves correspond to optimal answer sets of $P_2\cup\fix{P_1}{M_1}$ that either satisfy or not the program $C$ according to $\overline{\Box}$. 
Since $C$ is stratified with hard constraints, it admits one answer set~\cite{DBLP:journals/csur/DantsinEGV01} iff hard constraints are satisfied.   
Thus, the following transformation can be used to capture the incoherence of $C$.

\begin{definition}[Complement of stratified program]\label{def:complement_transformation}
Let $P$ be a stratified ASP program with hard constraints, then the \emph{complement} of $P$, denoted by $\neg P$, is obtained from $P$ by ($i$) transforming hard constraints $\leftarrow l_1,\ldots,l_n$ into rules of the form $v \leftarrow l_1,\ldots,l_n$; and ($ii$) adding an hard constraint of the form $\leftarrow\naf v$.

\end{definition}

\begin{prop}
Given a stratified ASP program $P$, its complement $\neg P$ is coherent iff $P$ is incoherent.
\end{prop}

At this point, if $\Box = \exists^{st}$ (resp. $\Box = \forall^{st}$) one might be tempted to define the counterexample program as $P_2 \cup \neg C \cup \fix{P_1}{M_1}$  (resp. $P_2 \cup C \cup \fix{P_1}{M_1}$). 
However, an optimal answer set of such a counterexample program may correspond to a non-optimal answer set of $P_2\cup\fix{P_1}{M_1}$ that violates (resp. satisfies) $C$, which does not correspond to a countermove for $\overline{\Box}$ to $M_1$.
To address this issue, we need to relax hard constraints from $C$ into weak constraints with a lower priority level w.r.t. weak constraints in $P_2$.

\begin{definition}[Relaxed program]\label{def:relaxed_transformation}
Let $P$ be a stratified program with hard constraints, $l$ be an integer, and $unsat$ be a fresh atom not appearing anywhere else. Then $relaxed(P,l)$ is defined as:
    \[
    relaxed(P,l) = \left\{\begin{array}{lllr}
        H_r     & \leftarrow    & B_r.          & \forall r \in P s.t. H_r \ne \emptyset\\
        unsat   & \leftarrow    & B_r.          & \forall r \in P s.t. H_r = \emptyset\\
                & \leftarrow^w  & unsat. [1@l]  & \\
     \end{array}\right\}
    \]
    
\end{definition}
Intuitively, $relaxed(\cdot, \cdot)$ is similar to $\neg C$ but the fresh atom introduced as head of the hard constraints of $P$ is used to assign a penalty if $P$ is incoherent. 
As a result, $relaxed(\cdot, \cdot)$ has the following property, which is fundamental in the computation of countermoves.
\begin{prop}
Given a stratified ASP program $P$ with hard constraints, and an integer $l$, $relaxed(P,l)$ is always coherent.
\end{prop}
Thanks to the above property, the program $relaxed(\neg C,l)$ (resp. $relaxed(C,l)$) can be combined with $P_2$ in such a way that rules in $\neg C$ (resp. $C$) do not filter out any optimal answer set of $P_2\cup\fix{P_1}{M_1}$. Thus, we are now ready to define the countermove program.

\begin{definition}[Countermove program]\label{def:countermove_program}
Let $\Pi$ be a \TWOASPQW of the form (\ref{eq:2_aspqw}) and $l_{min}$ be the smallest priority level among weak constraints in $P_2$, then the \emph{countermove program} for $\Pi$ is $ctr(\Pi) = P_2 \cup relaxed(\neg C,l_{min}-1)$ if $\Box = \exists^{st}$; otherwise $ctr(\Pi) = P_2 \cup relaxed(C,l_{min}-1)$.
\end{definition}
Note that, the above definition aligns with the countermove definition from Section~\ref{sec:game_based_2aspqw}. Specifically, when $\Box = \exists^{st}$, a countermove must violate $C$, so the countermove program incorporates rules from $relaxed(\neg C, l_{min}-1)$. Conversely, when $\Box = \forall^{st}$, a countermove must satisfy $C$, so the countermove program incorporates rules from $relaxed(C,l_{min}-1)$.
Consequently, countermoves to a move for $\Box$ are given by the optimal answer sets of the counterexample program. 
\begin{prop}\label{prop:ctr_move}
\updated{
Let $\Pi$ be a \TWOASPQW program of the form~(\ref{eq:2_aspqw}) and $M_1 \in OptAS(P_1)$ be a move for $\Box$. There exists $M_2 \in OptAS(ctr(\Pi)\cup \fix{P_1}{M_1})$ such that $unsat \notin M_2$ if and only if $M_2|_{\heads{P_2}}$ is a countermove to $M_1$ for $\overline{\Box}$.
}
\end{prop}

\subsection{Refining abstractions in \TWOASPQW}
The next step in designing our CEGAR-based approach for \TWOASPQW is to define a transformation for refining the abstraction according to known countermoves.
Specifically, let $M_1$ be a move for $\Box$ and $M_2$ be a countermove for $\overline{\Box}$, then the goal of the refinement, namely $Ref(\Pi, M_2)$, is to obtain a set of rules such that optimal answer sets of $P_1 \cup Ref(\Pi, M_2)$ correspond to moves $M_1'$ for $\Box$ such that $M_2$ is not a countermove to $M_1'$ for $\overline{\Box}$. 
Note that, $M_2$ is not a countermove to $M_1'$ for $\overline{\Box}$ if one of the these condition holds:
\begin{enumerate}
    \item there is no $M_2' \in AS(P_2\cup\fix{P_1}{M_1'})$ such that $M_2'|_{\heads{P_2}} = M_2$\label{cond:no_ans};
    \item there exists $M_2' \in AS(P_2\cup\fix{P_1}{M_1'})$ such that $M_2'|_{\heads{P_2}} = M_2$ but $M_2'\notin OptAS(P_2\cup\fix{P_1}{M_1'})$ (i.e. $M_2'$ is dominated by some answer set of $P_2\cup\fix{P_1}{M_1'}$);\label{cond:no_opt_ans}
    \item there exists $M_2' \in OptAS(P_2\cup\fix{P_1}{M_1'})$ such that $M_2'|_{\heads{P_2}} = M_2$ but $M_2'$ satisfies $C$ and $\Box = \exists^{st}$ (resp. violates $C$ and $\Box=\forall^{st}$).\label{cond:sat_c}
\end{enumerate}
Thus, we define some transformations encoding these conditions to refine our abstraction.

First of all, we need a predicate substitution function that maps each predicate $p$ to a fresh one of the form $p^{\alpha}_\beta$, where $\alpha$ and $\beta$ are strings, to obtain rules over a fresh signature for each discovered countermove.
More in detail, let $L$ be a set of literals and $\epsilon$ be an ASP expression (i.e. rule, program, etc.), then $\sigma^{\alpha}_\beta(L,\epsilon)$ denotes the expression obtained from $\epsilon$ by mapping each positive (resp. negative) literal $p(\vec{t}) \in L$ (resp. $\naf p(\vec{t}) \in L$) with $p^{\alpha}_\beta(\vec{t})$ (resp. $\naf p^{\alpha}_\beta(\vec{t})$).
Now, we are ready to define the rules which verify Condition~\ref{cond:no_ans}. 
\begin{definition}[Check Answer Set]\label{def:same_transformation}
    Let $P$ be an ASP program and $M$ be an interpretation, then $checkAS(P,M)$ is:
    \[
    checkAS(P, M) = \left\{\begin{array}{lr}
    \sigma_M^-(\heads{P},\{a \leftarrow \vert a \in M\})&\\
    \sigma_M^+(\heads{P}, \sigma_M^-(\overline{\heads{P}},r)) & \forall r \in P, H(r)\neq \emptyset\\
    fail_M \leftarrow \sigma_M^+(\heads{P}, \sigma_M^-(\overline{\heads{P}},B(r))) & \forall r \in P, H(r)=\emptyset\\
    fail_M \leftarrow p(\vec{t})_M^+,\naf p(\vec{t})_M^- & \forall\ p(\vec{t}) \in \heads{P} \\
    fail_M \leftarrow p(\vec{t})_M^-,\naf p(\vec{t})_M^+ & \forall\ p(\vec{t}) \in \heads{P}\\
    as_M \leftarrow \naf fail_M\ & \
    \end{array}\right.
    \]
\end{definition}
Basically, $checkAS(P, M)$ encodes the GL-reduct of a program $P$ w.r.t. an interpretation $M$ and checks whether $M$ is a $\subset$-minimal model of the reduct. 
Specifically, $M$ is encoded as facts over a fresh \emph{negative signature} (i.e., each atom $p(\vec{t}) \in M$ is represented by a fact $p^{-}_M(\vec{t}) \leftarrow$).
Atoms over the negative signature are used to rewrite negative body literals, yielding the GL-reduct of $P$ w.r.t. $M$; whereas head atoms and positive body literals are mapped to a fresh \emph{positive signature} (i.e., $p(\vec{t})$ mapped to $p^{+}_M(\vec{t})$).
Finally, if the true atoms over positive and negative signatures coincide then $M$ is an answer set of $P$.

We can now focus on Condition~\ref{cond:no_opt_ans} which requires checking answer set optimality. To this end, we define a transformation to compute the cost of answer sets of an ASP program.
\begin{definition}[Cost Program]\label{cost_transformation}
    Let $P$ be an ASP program, $L$ be the set of priority levels in $\mathcal{W}(P)$, and $v_P$ and $cl_P$ are fresh predicates not appearing in $P$. Then, $cost(P)$ is defined as:
    \[
    cost(P) = \left\{\begin{array}{lclr}
    v_P(w,l,\vec{t}) & \leftarrow & l_1, \ldots, l_n & \forall  \leftarrow^{\omega} l_1, \ldots, l_n [w@l,\vec{t}] \in P\\
    cl_P(\updated{T},l) &\leftarrow & \#sum\{C,\vec{t}: v_P(C,l,\vec{t})\}=\updated{T} & \forall l \in L\\
    \end{array}\right.
    \]
\end{definition}
Intuitively, weak constraints of the form $\leftarrow^{\omega} l_1, \ldots, l_n [w@l,\vec{t}]$ can be rewritten into rules whose head is a fresh atom encoding violation tuples (i.e. $(w,l,\vec{t})$); and introduces a rule for each level $l$ that sums up the the weights of the violation tuples at level $l$.
To verify the optimality of an answer set $M \in AS(P)$, it is necessary to compare its cost with that of every $M' \in AS(P)$.
To this end, we use $clone(P)=\sigma_{clone}(\heads{P}\cup\overline{\heads{P}}, P)$, which clones the program $P$, enabling the comparison between the cost of $M$ and the cost of all answer sets of the cloned program.

\begin{definition}[Dominated program]\label{dominated_transformation}
    Let $P_1$ and $P_2$ be two ASP programs, then $checkDom(P_1,P_2)$ is the following program: 
    \[
    \begin{array}{lclr}
    diff_{P_1,P_2}(L) & \leftarrow & cl_{P_1}(C1,L), cl_{P_2}(C2,L), C1 \neq C2. & \\
    hasHigher_{P_1,P_2}(L) & \leftarrow & diff_{P_1,P_2}(L), diff_{P_1,P_2}(L1), L<L1.& \\
    highest_{P_1,P_2}(L) & \leftarrow & diff_{P_1,P_2}(L), \naf hasHigher_{P_1,P_2}(L).& \\
    dom_{P_1,P_2} & \leftarrow & highest_{P_1,P_2}(L), cl_{P_1}(C1,L), cl_{P_2}(C2,L), C2 < C1 & \\
    \end{array}
    \]
    where $cl_{P_1}$ and $cl_{P_2}$ are, respectively, the predicates introduced by $cost(P_1)$ and $cost(P_2)$; whereas $diff_{P_1,P_2}$, $hasHigher_{P_1,P_2}$, $highest_{P_1,P_2}$, and $dom_{P_1,P_2}$ are fresh predicates.
\end{definition}
Intuitively, let $M \in AS(P)$ and $M_{c}\in AS(clone(P))$, then $checkDom(P, clone(P))$ can be used to compare the cost of $M$ and $M_{c}$.
More precisely, $cost(P)$ and $cost(clone(P))$ compute, respectively, the cost of $M$ and $M_c$, for each level $l$, as atoms of the form $cl_P(C,l)$ and $cl_{clone(P)}(C,l)$.
Thus, rules from $dom(P,clone(P))$ derive $dom_{P,clone(P)}$ if and only if $M$ is dominated by $M_{c}$, which means $M$ is not an optimal answer set of $P$.

Finally, we can focus on Condition~\ref{cond:sat_c}. 
Recall that Condition~\ref{cond:sat_c} requires to verify the coherence (resp. incoherence) of the program $C$. To this end, it is possible to leverage $relaxed(C,l)$ (resp. $relaxed(\neg C,l)$), as we have seen for counterexample program.

Now that, we have transformations for each condition (i.e., Condition~\ref{cond:no_ans},~\ref{cond:no_opt_ans}, and~\ref{cond:sat_c}), we need to control the activation of the different transformations in the refinement process. 
\begin{definition}[Controlled program]\label{controlled_program_transformation}
    Let $P$ be a program and $l$ a literal s.t. $at(l) \notin B_P$, then $or(P,l) = \{H_r \leftarrow B_r, l \mid r \in P\}$.
\end{definition}
Intuitively, $or(P,l)$ controls the activation of $P$ according to the literal $l$.
If $l$ is true, it can be removed from rules in $or(P,l)$, obtaining back the program $P$; otherwise rules in $or(P,l)$ are satisfied as $l$ falsifies rules' body.
We can now define the refinement program. 
\begin{definition}[Refinement Program]\label{def:refinement_program}
Let $\Pi$ be a \TWOASPQW program of the form~(\ref{eq:2_aspqw}), $M$ be a move for $\Box$ and let $CE$ be a countermove to $M$ for $\overline{\Box}$. The \textit{refinement program} is defined as follows:
     \[
    ref(\Pi, CE) = \left\{\begin{array}{lr}
        checkAS(P_2,CE) & \\
        \leftarrow^w as_{CE}\ [1@l_{min}-1] & \\
        or(cost(P_2^{CE}),as_{CE}) & \\
        or(clone(\mathcal{R}(P_2)),as_{CE}) & \\
        or(cost(clone(P_2)),as_{CE}) & \\
        or(checkDom(P_2^{CE}, clone(P_2)),as_{CE}) & \\
        \leftarrow^w as_{CE}, \naf dom_{P_2^{CE},clone(P_2)}\ [1@l_{min}-2] & \\
        or(\sigma_{CE}^-(lits(P_2)\cup lits(C'), C'),as_{CE}) & \\
    \end{array}\right\}
    \]
where $P_2^{CE} = \sigma_{CE}^-(lits(P_2),P_2)$, $lits(P_2) = \heads{P_2}\cup\overline{\heads{P_2}}$, $as_{CE}$ is the predicate introduced by $checkAS(P_2, CE)$, $C' = relaxed(C,l_{min}-3)$ if $\Box=\exists^{st}$; otherwise $C' = relaxed(\neg C,l_{min}-3)$, $lits(C') = \heads{C'}\cup\overline{\heads{C'}}$, and $l_{min}$ is the smallest priority level among weak constraints of $P_1$.
\end{definition}

Let $\Pi$ be a \TWOASPQW of the form~(\ref{eq:2_aspqw}), $M$ be a move for $\Box$ and $CE$ be a countermove to $M$ for $\overline{\Box}$, then we aim at computing a new move $M'$ for $\Box$ such that $CE$ is not a countermove to $M'$ for $\overline{\Box}$.
To this end, we use $ref(\Pi,CE)$ to check Conditions \ref{cond:no_ans}-\ref{cond:sat_c}. 
More precisely, rules in $checkAS(P_2,CE)$ encode the reduct of $P_2$ w.r.t. $CE$, and derive the atom $as_{CE}$ iff there exists $M_2 \in AS(P_2\cup\fix{P_1}{M'})$ such that $M_2|_{\heads{P_2}} = CE$ (i.e., Condition~\ref{cond:no_ans} is not satisfied).
Thus, the weak constraint $\leftarrow^{\omega} as_{CE}\ [1@l_{min}-1]$ expresses the preference to satisfy Condition~\ref{cond:no_ans}.
If Condition~\ref{cond:no_ans} is not satisfied, then there exists $M_2 \in AS(P_2\cup\fix{P_1}{M'})$ such that $M_2|_{\heads{P_2}} = CE$ and so $as_{CE}$ is derived as true.
At this point, $as_{CE}$ activates the following blocks which check Conditions~\ref{cond:no_opt_ans} and~\ref{cond:sat_c}.
To check Condition~\ref{cond:no_opt_ans}, $or(cost(P_2^{CE}),as_{CE})$ computes the cost of $M_2$ w.r.t. weak constraints in $P_2$; $or(clone(\mathcal{R}(P_2)),as_{CE})$ clones the program $P_2$ to compare pair of answer set of $P_2\cup\fix{P_1}{M'}$; $or(cost(clone(P_2)),as_{CE})$ computes the cost of an answer set of the cloned program; and finally $or(checkDom(P_2^{CE}, clone(P_2)),as_{CE})$ derive the atom $dom_{P_2,clone(P_2)}$ iff $M_2$ is not optimal as $M_2$ is dominated by some answer set of the cloned program.
As a result, the weak constraint $\leftarrow^w as_{CE}, \naf dom_{P_2^{CE},clone(P_2)}\ [1@l_{min}-2]$ prefers answer sets of the cloned program that dominate $M_2$ (i.e., $M_2$ is not optimal).
If $M_2$ cannot be dominated by any answer set of the cloned program then $M_2$ is optimal and so $or(\sigma_{CE}^-(lits(P_2)\cup lits(C'), C'),as_{CE})$ checks  Condition~\ref{cond:sat_c}. 
In particular, if $\Box = \exists^{st}$ (resp. $\Box = \forall^{st}$), then the atom $unsat$ is derived iff $M_2$ does not satisfy $C$ (resp. $\neg C$), which means Condition~\ref{cond:sat_c} is not satisfied.
Thus, the weak constraint $\leftarrow^{\omega} unsat, as_{CE}\ [1@l_{min}-3]$ in $C'$ models the preference of satisfying $C$ (resp. $\neg C$) (i.e. Condition~\ref{cond:sat_c} is satisfied).
Thus, if there exists an optimal answer set of $P_1\cup ref(\Pi,CE)$ that satisfies at least one weak constraint in $ref(\Pi,CE)$ then this corresponds to a move $M'$ that does not admit $CE$ as countermove. 
A detailed example is provided in Appendix~\ref{sec:program_transform_ref}.

\section{Computing Optimal Quantified Answer Sets}\label{sec:opt_qas}
Computing optimal (quantified) answer sets requires algorithms that make multiple oracle calls to search for such answer sets~\cite{DBLP:journals/logcom/AlvianoD0R20,DBLP:conf/padl/AzzoliniLMR25}. 

\paragraph{Quantified Answer Set Computation}
The computation of quantified answer sets in presence of weak constraints can be obtained by plugging the transformations defined in previous section in the \emph{CEGAR for \TWOASPQ} algorithm by \citeN{Cuteri_Mazzotta_Ricca_2026}. 
In particular, we update both counterexample search and refinement procedure with those introduced in Section~\ref{sec:cegar_for_2aspqw} (i.e., Definitions~\ref{def:countermove_program} and~\ref{def:refinement_program}), and update the conditions of existence of a winning move accordingly.
Indeed, in the original algorithm it was sufficient to check for existence of an answer set of the refined or counterexample programs, \updated{whereas} here we have to look at the cost of their optimal answer sets. More in detail, at the least three levels the cost 1 if no winning exists; whereas for the counterexample program at the last level the cost is 1 if the current move is winning (full algorithm in~\ref{sec:extended_cegar}). 

\paragraph{Upper-bound improving}
Optimal quantified answer sets can be computed starting from a quantified answer set and iteratively searching for better ones until the incoherence is met~\cite{DBLP:journals/logcom/AlvianoD0R20}. 
A quantified answer set $M$ of the input program $\Pi$ is computed applying CEGAR for \TWOASPQW, and the cost of $M$ is our initial upper bound. Note that if $M$ does not exist, no optimal quantified answer set exists and the search stops immediately. 
Next, to improve on this bound a constraint is added to $\Pi$ to enforce a preference for answer sets of $P_1$ (i.e. candidate quantified answer sets) with a lower cost, and the solver is called again. The process repeats until the upper bound cannot be improved anymore, and the last quantified answer set is optimal.

\paragraph{Lower-bound improving}
The core idea of these strategies is to fix an initial lower-bound cost and, if no answer set within this bound exists, iteratively relax the program to raise the bound until an optimal solution is obtained.
Traditional ASP solvers start by treating all weak constraints as hard, aiming for a zero-cost lower bound. If no answer set exists, unsatisfiable cores guide the progressive relaxation of the program, incrementally raising the lower bound (roughly  admitting some weak constraint must be violated) until an optimal solution is found~\cite{DBLP:journals/logcom/AlvianoD0R20}. 
Unluckily, this strategy cannot be ported as it is in our setting, since a notion of unsatisfiable core has never been defined for \ASPQW. 
Thus, we propose alternative ways for targeting a lower bound, and also for relaxing the program so that the lower bound improves iteratively until the optimum is found.
Intuitively, we aim to compute an answer set of $P_1$ that is optimal with respect to the global weak constraints. This is achieved by adding the global weak constraints $C^{\omega}$ to $P_1$ with the lowest priority and searching for the optimum answer sets of  the resulting program $P_1^{lower}$.
Note that, an optimum answer set of $P_1^{lower}$ is a reasonable lower bound candidate, since it is either an optimum for $\Pi$ or it does not satisfy the subsequent quantifiers. 
In the latter case, following the CEGAR approach, an oracle call is used to compute a countermove, and the corresponding refinement is  added to $P_1^{lower}$. 
This enables the next iteration to search for a new candidate optimal answer set. 
The procedure repeats until a winning move is found, incrementally improving the lower bound by discarding candidates that are not quantified answer sets of $\Pi$.
To ensure the correctness of the approach the global weak constraints $W$ are moved at the lowest priority levels w.r.t. both local weak constraints in $P_1$ and weak constraints that would be added to $P_1^{lower}$ by the refinement procedure.
More precisely, the initial abstraction is defined as $P_1 \cup W$, where 
$W = \{\leftarrow^\omega l_1,\ldots,l_k\ [w@\lambda+l,\vec{t}] \mid \leftarrow^\omega l_1,\ldots,l_k\ [w@l,\vec{t}] \in C^{\omega}\}$, with $\lambda$ being an integer such that levels from $C^{\omega}$ are remapped to strictly lower priority levels w.r.t. $l_{min}\!-\!3$ (i.e., the smallest priority level added by the refinement), and $l_{min}$ being the lowest priority level in $P_1$.
In this way, the first winning move will be a quantified answer set which is also optimal.
Note that the proposed ordering of priority levels in the (refined) abstraction is essential: it first favors moves that admit no known countermove and only then prefers moves that are optimal with respect to $C^{\omega}$. Under this ordering, the first winning move obtained is an optimal quantified answer set. 

\section{\updated{Implementation and }Experiments}
We run an experimental campaign on an Intel(R) Xeon(R) CPU E7-8880 v4 @ 2.20GHz, running Debian GNU/Linux 12, with memory and CPU (i.e., user+system) limited to 8GB and 800s.
Benchmarks and executables are available at
\url{https://osf.io/gmnjx}

\updated{
\paragraph{Implementation}
The proposed approach was implemented on top of the \textsc{casper}~\cite{Cuteri_Mazzotta_Ricca_2026}. 
More in detail, \textsc{casper} is written in Python and uses \textsc{clingo}~\cite{DBLP:conf/iclp/GebserKKOSW16} as an oracle for computing moves and countermoves.
For non-alternating \TWOASPQW programs, our implementation applies the transformations proposed by~\citeN{DBLP:journals/tplp/MazzottaRT24} for removing weak constraints, as the \TWOASPQW game is defined for alternating programs. Specifically, non-alternating \TWOASPQW programs are translated into \TWOASPQ program with two alternating quantifiers and no local weak constraints.
}
\paragraph{Benchmarks}
We considered several benchmarks from diverse \ASPQ applications ~\cite{DBLP:journals/tplp/FaberMR23,DBLP:conf/padl/AzzoliniLMR25,DBLP:conf/ijcai/AzzoliniMRR25}: Propositional Abduction Problem (PAP)~\cite{DBLP:journals/jacm/EiterG95}, Minmax Clique (MMC)~\cite{cao1995minimax}, Max Term Deletion (MTD)~\cite{schaefer2002completeness}, Most Probable Explanation in Probabilistic ASP (MPE)~\cite{DBLP:conf/ijcai/AzzoliniMRR25}, and Clique Coloring (CC)~\cite{schaefer2002completeness}.
For PAP, we consider three reasoning tasks: \textsc{pap-opt}, \textsc{pap-rel}, and \textsc{pap-nec}, corresponding to computing cardinality-minimal solutions, and relevant and necessary hypotheses, respectively.
For MMC we considered both the decision and optimization variant of the problem, namely \textsc{mmc-bound} and \textsc{mmc-opt}. 
For MTD, we considered the optimization version.
Finally, for MPE we considered both \textsc{coloring} and \textsc{smokers} domains~\cite{DBLP:conf/ijcai/AzzoliniMRR25}.
\updated{For the CC benchmark, we considered graphs of varying size (from 10 to 120 nodes) and edge density (25\%, 50\%, and 75\%), generated according to the Erdős–Rényi model provided by NetworkX library (\url{https://pypi.org/project/networkx}). For each combination of size and density, we generated 10 instances.
For the other benchmarks, instances were drawn from previous experiments~\cite{Cuteri_Mazzotta_Ricca_2026,DBLP:conf/ijcai/AzzoliniMRR25,DBLP:conf/padl/AzzoliniLMR25}.}

\paragraph{Compared methods}
In our evaluation we compared the approach by~\citeN{DBLP:conf/padl/AzzoliniLMR25} implemented in \textsc{pyqasp}.
\updated{Note that, \textsc{pyqasp} implements the upper-bound improving algorithm on top of a rewriting in QBF and it uses \textsc{quabs}~(\url{https://github.com/ltentrup/quabs}) as backend solver.}
\updated{Then, we consider also the proposed techniques denoted as \textsc{casper-u} (upper-bound improving) and \textsc{casper-l} (lower-bound improving).}
\begin{table}[t]
    \scriptsize
    \begin{tabular}{@{\extracolsep{\fill}}cccrr|rrr|rrr}
        \multirow{2}{*}{Bench} & \multirow{2}{*}{\#inst} &\multirow{2}{*}{TT} & \multicolumn{2}{c|}{\textsc{casper-l}} & \multicolumn{3}{c|}{\textsc{casper-u}} & \multicolumn{3}{c}{\textsc{pyqasp}}\\
            \cline{4-11} 
            & & &  Sol. 	& Sum t.(s) 	&  Sol. 	& Sum t.(s) 	& \multicolumn{1}{c|}{\#opt} & Sol. 	& Sum t.(s) 	& \multicolumn{1}{c}{\#opt}\\
        \cline{1-11}
        \textsc{pap-opt}    & 294   & $opt$        & \textbf{123}	 & \updated{\textbf{6450.78}}	  & \updated{33}	          & \updated{9725.87}	           & \updated{367}	& 101           & 18076.72          & 410\\
        \textsc{mtd}	    & 80    & $opt$        & \updated{50}	             & \updated{1698.24}	          & \textbf{74}   & \updated{\textbf{3916.51}}   & 1184   & 48	        & 6451.46	        & 253\\
        \textsc{coloring}   & 94    & $opt$        & 11	             & \updated{896.93}	          & 8	          & \updated{114.35}	           & 105	& \textbf{25}   & \textbf{3621.92}	& 278\\
        \textsc{smokers}    & 99    & $opt$        & \textbf{99}	 & \updated{\textbf{83.54}}	  & 19	          & \updated{3130.97}	           & 3724	& 14	        & 1845.59	        & 152\\
        \textsc{mmc-opt}    & 45    & $opt$        & \textbf{44}	 & \updated{\textbf{3565.79}} 	  & 42	          & \updated{2912.66}           & 106    & 5	            & 904.65            &	5\\
        
        \textsc{pap-nec}	& 294	& $coh$        & \textbf{282}    & \updated{\textbf{2675.68}}	  & \textbf{282}  & \updated{\textbf{2675.68}}   & -      & -             & -                 & -\\
        \textsc{pap-rel}	& 294	& $coh$        & \textbf{294}    & \updated{\textbf{2053.67}}   & \textbf{294}  & \updated{\textbf{2053.67}}   & -      & -	            & -                 & -\\
        \textsc{mmc-bound}  & 225	& $coh$        & \updated{\textbf{225}}    & \updated{\textbf{919.39}}	  & \updated{\textbf{225}}  & \updated{\textbf{919.39}}	   & -      & -	            & -                 & -\\
        \updated{\textsc{cc}}  & \updated{420}	& \updated{$coh$}        & \updated{\textbf{370}}    & \updated{\textbf{5555.22}}	  & \updated{\textbf{370}}  & \updated{\textbf{5555.22}}	   & \updated{-}      & \updated{-}	            & \updated{-}                 & \updated{-}\\
        
        \cline{1-11}
    \end{tabular}

    \caption{Overall results}
    \label{tab:results}
\end{table}

\paragraph{Results}
Obtained results are summarized in Table~\ref{tab:results} which reports, for each benchmark, the number of instances (\#inst), the type of task (TT) ($opt$ for optimal answer set and $coh$ for coherence), and, for each system, the number of solved instances (Sol.), total execution time (Sum t,(s)), and number of optimization steps (\#Opt.). This latter is the sum of the number of quantified answer sets found to reach the solution; it is omitted for lower-bound improving, and is not meaningful for local-weak-constraints-only problems.

We first focus on the benchmarks containing only local weak constraints, namely: \textsc{pap-nec}, \textsc{pap-rel}, and \textsc{mmc-bound}.
Observe that, in these cases, \textsc{pyqasp} could not be run (it does not support local weak constraints), and \textsc{casper-l} and \textsc{casper-u} clearly coincide, since the first winning move is the solution.
For these benchmarks, nearly all instances were solved by \textsc{casper} within 2 minutes (only \updated{12} of 813 timed out), confirming effectiveness of the systems herein introduced.
\updated{To further assess the scalability of \textsc{casper}, we considered the CC benchmark. In particular, we compare \textsc{casper} with an enumeration-based solver, denoted as \textsc{nested-aspq}, which evaluates programs by enumerating answer sets of each subprogram according to the \TWOASPQW semantics.
Figure~\ref{fig:scalability} reports the execution times of the systems for graphs with fixed edge density. Instances are sorted by increasing graph size, and each point $(x, y)$ represents the average runtime $y$ (over 10 generated instances) required by a system on graphs with $x$ nodes and a given edge density.

\begin{minipage}{.49\textwidth}
    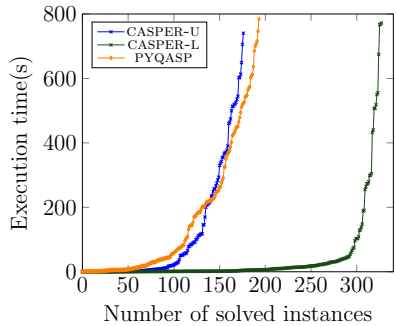
\begin{figure}[H]
\begin{tikzpicture}[scale=0.6]
    \begin{axis}[
        tick label style={font=\Large},
        legend style={nodes={scale=0.51, transform shape}, font=\huge},
        x label style={at={(axis description cs:0.5,-0.01)},anchor=north},
        y label style={at={(axis description cs:-0.05,.5)},anchor=south},
        xlabel={{\Large Number of solved instances}},
        ylabel={{\Large Execution time(s)}},
        xlabel style={yshift=-16pt},
        ylabel style={yshift=18pt},
        xmin=0, xmax=340,
        xtick={0,50,100,150,200,250,300},
        ymin=0, ymax=800,
        ytick={0,200,400,600,800},
        ymode=normal,
        legend pos=north west
        ]
        
        \addplot[color=blue, mark=x, mark size = 1.4pt, line width=0.6pt] 
        table{overall_global_casper.bash_time.dat};
        \addlegendentry{\textsc{casper-u}}
        
        \addplot[color=darkgreen, mark=x, mark size = 1.4pt, line width=0.6pt] 
        table{overall_global_casper_lower.bash_time.dat};
        \addlegendentry{\textsc{casper-l}}
        
        \addplot[color=orange, mark=diamond, mark size = 1.1pt, line width=0.6pt] 
        table{overall_global_pyqasp_quabs_weak.bash_time.dat};
        \addlegendentry{\textsc{pyqasp}}
        
    \end{axis}
\end{tikzpicture}
        \caption{\updated{Overall execution time - $opt$}}
        \label{fig:overall_time}
    \end{figure}
    
\end{minipage}
\begin{minipage}{.49\textwidth}
    \begin{figure}[H]
\begin{tikzpicture}[scale=0.6]
    \begin{axis}[
        tick label style={font=\Large},
        legend style={nodes={scale=0.51, transform shape}, font=\huge},
        x label style={at={(axis description cs:0.5,-0.01)},anchor=north},
        y label style={at={(axis description cs:-0.05,.5)},anchor=south},
        xlabel={\Large{Number of nodes}},
        ylabel={\Large{Execution time(s)}},
        xlabel style={yshift=-16pt},
        ylabel style={yshift=26pt},
        xmin=0, xmax=120,
        xtick={0,40,80,120},
        ymin=0, ymax=1600,
        ytick={0,400,800,1200,1600},
        ymode=normal,
        legend pos=north west
        ]
        
        \addplot[color=blue, mark=x, mark size = 1.4pt, line width=0.6pt] 
        table{synt_casper_0.25_time.dat};
        \addlegendentry{\textsc{casper-0.25}}
        
        \addplot[color=cyan, mark=x, mark size = 1.4pt, line width=0.6pt] 
        table{synt_casper_0.50_time.dat};
        \addlegendentry{\textsc{casper-0.50}}
        
        \addplot[color=teal, mark=x, mark size = 1.4pt, line width=0.6pt] 
        table{synt_casper_0.75_time.dat};
        \addlegendentry{\textsc{casper-0.75}}
        
        \addplot[color=brown, mark=diamond, mark size = 1.1pt, line width=0.6pt] 
        table{synt_nested_0.25_time.dat};
        \addlegendentry{\textsc{nested-aspq-0.25}}

        \addplot[color=orange, mark=diamond, mark size = 1.1pt, line width=0.6pt] 
        table{synt_nested_0.50_time.dat};
        \addlegendentry{\textsc{nested-aspq-0.50}}

        \addplot[color=purple, mark=diamond, mark size = 1.1pt, line width=0.6pt] 
        table{synt_nested_0.75_time.dat};
        \addlegendentry{\textsc{nested-aspq-0.75}}

    \end{axis}
\end{tikzpicture}
        \caption{\updated{Solving time for \textsc{CC}}}
    \label{fig:scalability}
    \end{figure}
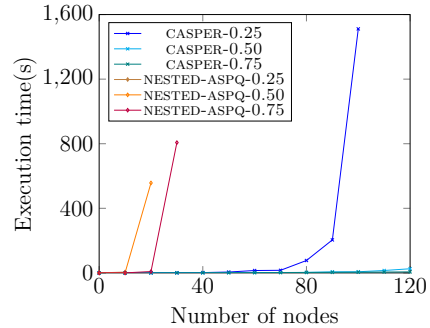
\end{minipage}\smallskip

The generated instances are hard to solve and \textsc{nested-aspq} is unable to scale beyond graphs with 30 nodes, regardless of the edge density.
In contrast, \textsc{casper} exhibits significantly better scalability, solving instances with up to 120 nodes. Furthermore, the edge density has a noticeable impact on \textsc{casper} runtime. For dense graphs, the runtime is considerably small, as dense graphs admits \textit{few} maximal cliques and so the number of possible counterexample is reduced. Conversely, for sparse graphs, the runtime increases, due to the larger number of maximal cliques, which in turn leads to a higher number of counterexamples to be explored.}  

Let's now shift the attention to optimum quantified answer set search. 
For \textsc{mtd} and \textsc{coloring}, the upper-bound improving strategy, implemented in \textsc{pyqasp} and \textsc{casper-u}, is preferred over the lower-bound improving strategy adopted by \textsc{casper-l}. 
The ASP-based \textsc{casper-u} is preferable to the QBF-based \textsc{pyqasp} in \textsc{mtd}, whereas the opposite holds for \textsc{coloring}.
Diving in the details, \textsc{pyqasp} could compute \updated{for \textsc{coloring}} quantified answer sets that are closer to the optimum so requiring few optimization steps; on the other hand \textsc{casper-u} shows a faster computation of quantified answer sets leading to better performance. On the other hand, \textsc{casper-l} likely computes many locally optimal moves that are not quantified answer sets, thus resulting slower than upper-bound improving alternatives. 
Conversely, for \textsc{pap-opt}, \textsc{mmc-opt}, and \textsc{smokers}, \textsc{casper-l} solves substantially more instances than the others (146 more than \textsc{pyqasp} and \updated{172} more than \textsc{casper-u}), since locally optimal moves frequently correspond to optimal solutions.

Overall, \textsc{casper-l} solves \updated{134} more instances than \textsc{pyqasp}, and \updated{151}  more instances than \textsc{casper-u} with a lower average runtime. Figure~\ref{fig:overall_time} reports a traditional cactus plot on all the instances, confirming \textsc{casper-l} is the most effective system overall. 

\section{Related Work}
Many ASP extensions have been proposed for modeling hard combinatorial optimization problems. The standard ASP construct for expressing optimization problems is weak constraints~\cite{DBLP:journals/tkde/BuccafurriLR00}, which is equivalent to optimize statements~\cite{DBLP:series/synthesis/2012Gebser}. 
The \ASPQW language is based on the same construct, but expands the modeling capabilities of ASP in the entire PH.
For alternative formalism to \ASPQ, such as stable-unstable~\cite{DBLP:journals/tplp/BogaertsJT16} and quantified ASP~\cite{DBLP:journals/tplp/FandinnoLRSS21}, we are not aware of any extension considering  optimization statements explicitly. Optimization in ASP might also be handled within the \textit{asprin} framework~\cite{DBLP:journals/ai/BrewkaDRS23}, which however targets preference modeling. An in-depth comparison of \ASPQ with these formalism was provided by \citeN{DBLP:journals/tplp/AmendolaRT19} and \citeN{DBLP:journals/tplp/FandinnoLRSS21}.

Among \ASPQ systems, we mention \textsc{qasp}~\cite{DBLP:conf/lpnmr/AmendolaCRT22}, \textsc{pyqasp}~\cite{DBLP:journals/tplp/FaberMR23}, and \textsc{casper}~\cite{Cuteri_Mazzotta_Ricca_2026}. 
The first two are based on a translation of \ASPQ in QBF, they support an arbitrary number of quantifiers, but originally lacked support for weak constraints.
Recently, \textsc{pyqasp} was extended to handle global weak constraints using an upper-bound improving strategy~\cite{DBLP:conf/padl/AzzoliniLMR25}, though local weak constraints remain unsupported.
\textsc{casper}~\cite{Cuteri_Mazzotta_Ricca_2026} is the only CEGAR-based \ASPQ system and originally supported only \TWOASPQ programs without weak constraints. 
This paper extends \textsc{casper} to support local and global weak constraints, 
yielding the first implementation capable of evaluating \TWOASPQW programs. 
Solvers for \ASPQW exploit optimization strategies used in ASP solvers~\cite{DBLP:journals/logcom/AlvianoD0R20}, however our lower-bound improving approach departs from existing methods by realizing improvements via abstraction refinement.

\section{Conclusion}
This paper focuses on \TWOASPQW, the class of \ASPQW programs with two quantifiers, for which tight complexity bounds and concrete implementations were previously missing. We fill this gap by providing a detailed complexity analysis and establishing tight completeness results: coherence checking is complete for the second level of the PH (i.e., $\Sigma_2^P$ for existential programs and $\Pi_2^P$ for universal ones), while reasoning over optimal quantified answer sets is $\Delta_3^P$-complete.
Moreover, building on the CEGAR framework, we developed two optimization techniques based on lower- and upper-bound improvement strategies. Experimental results demonstrate the effectiveness of our approach, advancing the state of the art in \TWOASPQW solving. 

Future work includes extending both the complexity analysis and evaluation techniques to \ASPQW programs with an arbitrary number of quantifiers.

\bibliographystyle{acmtrans}
\bibliography{biblio}
\newpage
\appendix

\section{Extended Preliminaries}\label{sec:ext_prel}
\subsection{Complexity Classes}\label{sec:ext_prel_complexity}
In this section, we recall some basic definitions of complexity classes that are used to study the complexity of \TWOASPQW. For further details about $NP$-completeness and complexity theory we refer the reader to dedicated literature~\cite{DBLP:books/daglib/0072413}.  
We recall that the classes $\Delta_k^P$, $\Sigma_k^P$, and $\Pi_k^P$ of the polynomial time hierarchy (PH)\cite{DBLP:journals/tcs/Stockmeyer76} are defined as follows (rf. \citeN{DBLP:books/fm/GareyJ79}):
$$  \Delta_0^P = \Sigma_0^P = \Pi_0^P = P$$
and, for all $k>0$
$$  \Delta_{k+1}^P = P^{\Sigma_k^P},\ \ \Sigma_{k+1}^P = NP^{\Sigma_k^P},\ \ \Pi_{k+1}^P = coNP^{\Sigma_k^P},$$
where, $NP = \Sigma_1^P$, $coNP = \Pi_1^P$, and $\Delta_2=P^{NP}$. 

In general, $P^C$ (resp. $NP^C$) denotes the class of problems that can be solved in polynomial time on a deterministic (resp. nondeterministic) Turing machine with an oracle in the class $C$. 
Note that, the usage of an oracle $O\in C$ for solving a problem $\pi$ is referred to as a subroutine call, during the evaluation of $\pi$, to $O$. The latter is evaluated in a unit of time. 
\subsection{Answer Set Programming}\label{sec:ext_prel_asp}

\paragraph{\textbf{Syntax.}}
A term is a constant (i.e., an integer or a string starting with lowercase letter) or a variable (i.e., a string starting with uppercase letter).
A \textit{standard atom} is an expression of the form $p(\vec{t})$ with $\vec{t}=t_1,\ldots,t_n$ being a list of terms and $p$ being a \textit{predicate} of arity $n\ge0$. A standard atom $p(\vec{t})$ is \textit{ground} if all terms in $\vec{t}$ are constants. 
A \textit{standard literal} is either a standard atom $a$ or its negation $\naf a$.
An \textit{aggregate element} is a pair $t_1,\ldots,t_n : conj$ where $t_1, \ldots, t_n$ is a list of terms and $conj$ is a conjunction of standard literals. An \textit{aggregate atom} is of the form $f\{e_1, \ldots, e_n\}\prec t$, where $f \in \{\#count,\#sum\}$ is an \textit{aggregate function}, $\prec \in \{<, \le, >, \ge, =\}$ is a comparison operator, $t$ is a term called \textit{guard}, and $e_1,\ldots,e_n$ is a list of aggregate elements. An \textit{atom} can be either a standard atom or an aggregate atom. 
A \textit{literal} is either an atom (positive literal) or its negation (negative literal). The \textit{complement} of a literal $l$ is denoted by $\overline{l}$. The complement of a literal $l=a$ is $\naf a$, while the complement of a literal $l=\naf a$ is $a$. Given a set of literals $L$, $L^+$ (resp. $L^-$) denotes the set of positive (resp. negative) literals appearing in $L$.
A \textit{normal rule} is an expression of the form
$h \leftarrow l_1, \ldots, l_n$
where $h$ is standard atom referred to as \textit{head}, and $l_1,\ldots,l_n$ with $n\ge0$ is a conjunction of literals referred to as \textit{body}. A normal rule with empty body is called \textit{fact}, while a normal rule with empty head is called \textit{hard constraint} or \textit{strong constraint}.
A \textit{weak constraint} is an expression of the form
$\leftarrow^w l_1, \ldots, l_n\ [w@l,\vec{t}]$
where $l_1, \ldots, l_n$ are literals referred to as body, $w$ and $l$ are terms, and $\vec{t}= t_1, \ldots, t_m$ is a possibly empty list of terms.
A \textit{rule} $r$ is either a normal rule or a weak constraint.
Given a rule $r$, $H_r$ denotes the set of atoms appearing in the head of $r$ while $B_r$ denotes the set of atoms appearing in the body of $r$.
Given an expression $\epsilon$ (atom, rule, etc.), $\mathcal{V}(\epsilon)$ denotes the set of variables appearing in $\epsilon$.
For a normal rule $r$ the \textit{global variables} of $r$ are the variables appearing in $H_r$, in standard literals in $B_r$ or as the guard of aggregate literals of the form $f\{e_1, \ldots, e_n\}=V$.
A normal rule $r$ is \textit{safe} if ($i$) its global variables appear at least in one positive standard literal, and ($ii$) each variable appearing in an aggregate element $e$ is either a global variable or appears in some positive literal in $e$.
A weak constraint $r$ is \textit{safe} if it respects the safety condition for normal rules and variables in $\mathcal{V}(w) \cup \mathcal{V}(l) \cup \mathcal{V}(\vec{t})$ also appear in at least one positive standard literal in the body.
A \textit{program} $P$ is a set of safe rules.
Given a program $P$, $\mathcal{R}(P)$ and $\mathcal{W}(P)$ denote the sets of normal rules and weak rules appearing in $P$, respectively, while $\heads{P}$ denotes the set of atoms appearing as heads of rules in $P$.
Given an expression $\epsilon$ (atom, rule, program, etc.), $at(\epsilon)$ denotes the set of atoms appearing in $\epsilon$.
A \textit{choice rule} is of the form $\{e_1;\ldots;e_n\}\leftarrow l_1,\ldots, l_m$, where $e_1;\ldots;e_n$ are called \textit{choice elements} and $l_1,\ldots,l_m$ are literals. A choice element $e_i$ is of the form $a^i:b_1^i, \ldots, b_k^i$ with $a^i$ being a standard atom and $b_1^i, \ldots, b_k^i$ being standard literals. A program $P$ with choice rules can be rewritten into a program without choice rules. More in detail, each choice element $e_i$ can be seen as the pair of rules $a^i\leftarrow b_1^i, \ldots b_m^i, l_1, \ldots l_m, \naf na^i$ and $na^i \leftarrow b_1^i, \ldots b_m^i, l_1, \ldots l_m, \naf a^i$, with $na^i$ being a fresh predicate not appearing in $P$.

\paragraph{\textbf{Semantics.}}
Given an ASP program $P$, the \textit{Herbrand Universe}, $HU_P$, of $P$ is the set of all constants appearing in $P$; the \textit{Herbrand Base}, $B_P$, is the set of all possible ground atoms that can be constructed using predicates from $P$ and constants from $HU_B$; $ground(P)$ denotes the set of all ground rules that can be obtained from $P$ by proper substitutions of variables in $P$ with constants in $HU_P$.
An \textit{interpretation} $I \subseteq HB_P$ is a set of standard atoms. A ground standard literal $l=a$ (resp. $l=\naf a$) is true w.r.t. $I$ if $a \in I$ (resp. $a \notin I$), false otherwise.
A conjunction of standard literals $conj$ is true w.r.t. $I$ if all the literals in $conj$ are true w.r.t. $I$, false otherwise.
For a given set of aggregate elements $S=\{e_1; \ldots; e_n\}$, $eval(S,I)$ denotes the set of tuples $(t_1, \ldots, t_m)$ such that there exists an aggregate element $e_i \in S$ of the form $t_1, \ldots, t_m : conj$ and $conj$ is true w.r.t. $I$; $I(S)$ denotes the multi-set $[t_1 \vert (t_1, \ldots, t_m) \in eval(S,I)]$. A ground aggregate literal of the form $f\{e1; \ldots; e_n\}\prec t$ (resp. $\naf f\{e1; \ldots; e_n\}\prec t$) is true w.r.t. $I$ if $f(I({e_1; \ldots; e_n}))\prec t$ holds (resp. does not hold); otherwise is false.
An \textit{interpretation} $I$ is an \textit{answer set} of a program $P$ iff (i) $I$ is a model of $P$, namely for each rule $r \in ground(P)$ either the head of $r$ is true w.r.t. $I$ or the body of $r$ is false w.r.t. $I$; and (ii) $I$ is a minimal model of its FLP-reduct~\cite{DBLP:journals/ai/FaberPL11}.
Let $AS(P)$ be the set of answer set of a program $P$ then $P$ is coherent iff $AS(P)\neq \emptyset$. 
For a program $P$ and an interpretation $I$, let the set of weak constraint violations be 
$ws(P,I) = \{(w,l,T) \mid \ \leftarrow^{\omega} b_1,\ldots,b_m\ [w@l,\vec{t}] \in ground(P)$ and $b_1,\ldots,b_m$ are true w.r.t. $I\}$, then the cost function of $P$ is defined as
$\mathcal{C}(P,I,k) = \sum_{(w,k,T) \in ws(P,I)} w$ for every integer $k$.
Let $M_1,M_2 \in AS(P)$ then $M_1$ is \textit{dominated} by $M_2$ if there exists an integer $l$ such that $\mathcal{C}(P,M_1,l)>\mathcal{C}(P,M_2,l)$ and for each $l'>l$, $\mathcal{C}(P,M_1,l') = \mathcal{C}(P,M_2,l')$.
Let $M \in AS(P)$ then $M$ is an \textit{optimal answer set} iff $M$ is not dominated by any $M' \in AS(P)$. We denote by $OptAS(P)$ the set of optimal answer set of $P$.

\section{Full proof of Complexity Results}\label{sec:complexity_res}
The coherence problem for \ASPQW programs has been investigated in recent years~\cite{DBLP:journals/tplp/MazzottaRT24}, but several complexity-theoretic results are still missing.
\citeN{DBLP:journals/tplp/MazzottaRT24} studied the complexity of \ASPQW when weak constraints are used both globally (i.e. to define optimal quantified answer sets) and locally (i.e. within the different subprograms). In particular, the following results have been established. 
\begin{theorem}[\updated{Thm 2 by}~\citeN{DBLP:journals/tplp/MazzottaRT24}]
	The coherence problem of a 2-\ASPQ with weak constraints is in: $(i)$ $\Sigma_3^P$ for existential programs; $(ii)$ $\Pi_3^P$ for universal programs.
\end{theorem}
\begin{theorem}[\updated{Thm 3 by}~\citeN{DBLP:journals/tplp/MazzottaRT24}]
	The coherence problem of a 2-\ASPQ with weak constraints is hard for: $(i)$ $\Sigma_2^P$ for existential programs; $(ii)$ $\Pi_2^P$ for universal programs.
\end{theorem}
Starting from this results we further studied the complexity for 2-\ASPQW to provide complete results for such class of \ASPQW programs. 
More precisely, we focus on the general case of \ASPQW programs where weak constraints appears locally within each ASP programs. 
To present our result we first introduce some notation which will be instrumental in the proofs.
First of all, we recall the notion of splitting set by~\citeN{DBLP:conf/iclp/LifschitzT94}.
\begin{definition}
	Let $P$ be an ASP program and $U \subseteq B_P$. Then $U$ is a \textit{splitting set} for $P$ if for each rule $r \in P$ such that $\heads{r} \subseteq U$ then also $at(r) \subseteq U$. 
\end{definition}
Intuitively, a splitting set $U$ for a program $P$ allows to split the program $P$ into two programs referred to as \textit{bottom}, denoted by $b_U(P)$, and \textit{top}, denoted by $t_U(P)$. More precisely, $b_U(P)$ contains all those rules $r \in P$ such that $at(r)\subseteq U$, and $t_U(P)$ contains all the remaining rules, that is, $t_U(P) = P \setminus b_U(P)$.
Let $U$ and $X\subseteq U$ be two set of atoms $e_U(P,X)$ denotes the program obtained from rules $r \in P$ such that $B_r^+ \cap U \subseteq X$ and $B_r^-\cap X = \emptyset$, by removing all literals in $B_r$ whose atom is in $U$.
By exploiting the notion of splitting set, the task of computing an answer set of $P$ can be divided  in two stages. The following theorem formalizes such intuition.
\begin{theorem}[\citeN{DBLP:conf/iclp/LifschitzT94}]
	Let $P$ be a program and $U$ be a splitting set for $P$ then $M \in AS(P)$ if and only if $M = X\cup Y$ where $X \in AS(b_U(P))$ and $Y \in e_U(t_U(P),X)$.
\end{theorem}
\begin{example}
	Let $P$ be the following program:
	$$
	\begin{array}{l}
		a\leftarrow b, not\ c\\
		b\leftarrow c, not\ a\\
		c\leftarrow 
	\end{array}
	$$
	Then $U = \{c\}$ is a splitting set for $P$, with $b_U(P) = \{c\leftarrow\}$.
	Here the program $b_U(P)$ has only one answer set that is $X = \{c\}$ and the program $e_U(t_U(P),X)$ contains only the rule $b\leftarrow not\ a$.
	As a result, the program $e_U(t_U(P),X)$ has only one answer set that is $Y=\{b\}$, and so, $M = X \cup Y = \{c,b\}$ is the unique answer set of $P$.
\end{example}
In what follows, we leverage the notion of splitting set to obtain some relevant properties of the program transformation we are going to introduce.
The first program transformation guarantees the coherence of an ASP program.
\begin{definition}\label{def:sat_prg_app}
	Let $P$ be an ASP program then $sat(P)$ denotes the program:
	$$
	sat(P) = \left\{\begin{array}{lr}
		H_r \leftarrow B_r, not\ sat &  \forall r \in P\\
		\{sat\}\leftarrow& 
	\end{array}\right.
	$$
\end{definition}
Intuitively, for each ASP program $P$, the transformation $sat(P)$ prevents $P$ from being incoherent as it introduces a choice rule over the fresh atom $sat$ which controls the activation of the rules in $P$. If $sat$ is chosen as true, then all the rules of $P$ are trivially satisfied as $not\ sat$ appears in all rules' body. On the contrary, if $sat$ is chosen as false, then $sat(P)$ is coherent if and only if $P$ is coherent as $not\ sat$ can be removed from the body of each rule.
\begin{obs}
	Let $P$ be an program, then $AS(sat(P)) = \{\{sat\}\}\cup AS(P)$.
\end{obs}
\begin{example}
	Let $P$ and $sat(P)$ be the following programs:
	
	\begin{minipage}{.49\textwidth}
		$$
		P = \left\{\begin{array}{l}
			a \leftarrow not\ b  \\
			b \leftarrow not\ a  \\
			c \leftarrow not\ d  \\
			d \leftarrow not\ c  \\ 
		\end{array}\right\}
		$$    
	\end{minipage}
	\hfill
	\begin{minipage}{.49\textwidth}
		$$
		sat(P) = \left\{\begin{array}{l}
			\{sat\}\leftarrow\\
			a \leftarrow not\ b,not\ sat \\
			b \leftarrow not\ a,not\ sat \\
			c \leftarrow not\ d,not\ sat \\
			d \leftarrow not\ c,not\ sat \\ 
		\end{array}\right\}
		$$    
	\end{minipage}
	Here it is important to observe that the choice rule $\{sat\}\leftarrow$ is a shorthand for the normal rules $sat \leftarrow not\ nsat$ and $nsat\leftarrow not\ sat$ where $nsat$ is a fresh atom not appearing anywhere else. 
	In particular, $\{sat\}$ is trivially an answer set. On the other hand, if we pick $\{a,c\} \in AS(P)$ it is to see the $\{nsat, a, c\} \in AS(sat(P))$. 
	In general, since the atom $nsat$ does not appear anywhere else then it is customary in the literature to consider it as an hidden atom and so, with a slightly abuse of notation we can state that $AS(sat(P)) = \{\{sat\}\}\cup AS(P)$.    
\end{example}
Thanks to such a property, we can deal with incoherent ASP programs and preserve the \ASPQ semantics by proper program transformations. 
\begin{lemma}\label{lemma:prg_union_app}
	Let $P_1$ and $P_2$ be two ASP programs such that $\heads{P_2}\cap at(P_1) = \emptyset$, then
	$M \in AS(P_1 \cup sat(P_2))$ iff $M$ is of the form $M_1 \cup M_2$ where $M_1 \in AS(P_1)$ and $M_2\subseteq \heads{sat(P_2)}$ is either equal to $\{sat\}$ or $M_2$ is such that $M_1\cup M_2 \in AS(P_2\cup\fix{P_1}{M_1})$.
\end{lemma}
\begin{proof}
	Since $\heads{P_2}\cap at(P_1) = \emptyset$, then $U = at(P_1)$ is a splitting set of $P = P_1 \cup sat(P_2)$.
	Thus $M \in AS(P)$ iff $M = X \cup Y$, where $X \in AS(b_U(P)) = AS(P_1)$ and $Y \in AS(e_U(t_U(P),X)) = AS(e_U(sat(P_2),X))$.
	Let $X \in AS(P_1)$, then the program $e_U(sat(P_2),X)$ is obtained from rules in $sat(P_2)$ that have no body literals over $at(P_1)$ which are false w.r.t. to $X$ and by removing all the literals over atoms in $at(P_1)$. 
	Since the atom $sat$ does not appear in $P_1$, then the choice rule $\{sat\}$ remains as it is and the literal $not\ sat$ is not removed from any rule. As a result, $e_U(sat(P_2),X)$ is of the form:
	$$
	\begin{array}{lr}
		\{sat\} &  \\
		H_r \leftarrow B_r, not\ sat & \forall r \in e_U(P_2,X) 
	\end{array}
	$$
	Thus, $e_U(sat(P_2),X) = sat(e_U(P_2,X))$, and so $Y \in AS(sat(e_U(P_2,X)))$ iff $Y = \{sat\}$ or $Y \in AS(e_U(P_2,X))$. As a result, $M \in AS(P)$ iff $M = X \cup Y$ where $X \in AS(P_1)$ and $Y$ is either equal to $\{sat\}$ or $Y \in AS(e_U(P_2,X))$.
	
	In particular, if $Y \in AS(e_U(P_2,X))$, then $M = X \cup Y \in AS(P_2\cup \fix{P_1}{X})$ since $U$ is also a splitting set for $P_2' = P_2\cup \fix{P_1}{X}$ and so $M \in AS(P_2')$ iff $X \in AS(b_U(P_2')) = AS(\fix{P_1}{X}) = \{X\}$ and $Y\in AS(e_U(t_U(P_2'),X)) = AS(e_U(P_2,X))$.
	
	Thus, the thesis follows.
\end{proof}
\begin{lemma}\label{lemma:prg_union_weak_app}
	Let $P_1$ and $P_2$ be two ASP programs such that $\heads{P_2}\cap at(P_1) = \emptyset$ and $W_1$ be a set of weak constraints such that $at(W_1)\subseteq at(P_1)$ then
	$M \in OptAS(P_1\cup W_1 \cup sat(P_2))$ iff $M$ is of the form $M_1 \cup M_2$ where $M_1 \in OptAS(P_1\cup W_1)$ and $M_2\subseteq \heads{sat(P_2)}$ is either equal to $\{sat\}$ or $M_2$ is such that $M_1\cup M_2 \in AS(P_2\cup\fix{P_1}{M_1})$.
\end{lemma}
\begin{proof}
	Let $P = P_1 \cup W_1 \cup sat(P_2)$. From Lemma~\ref{lemma:prg_union_app}, $M\in AS(P)$ iff $M$ is of the form $M_1\cup M_2$, where $M_1 \in AS(P_1)$ and $M_2$ is either equal to $\{sat\}$ or $M_2$ is such that $M_1\cup M_2 \in AS(P_2\cup\fix{P_1}{M_1})$.
	Thus, by contradiction, let us assume that there exists $M = M_1 \cup M_2 \in OptAS(P)$ such that $M_1 \notin OptAS(P_1)$.
	At this point, we can construct an answer set $M' = M_1'\cup M_2' \in AS(P)$ such that $M_1 \in OptAS(P_1)$. Since $at(W_1)\subseteq B_{P_1}$ then the cost of $M$ and $M'$ is given, respectively, by $M_1$ and $M_1'$. Since $M_1\notin OptAS(P_1)$ and $M_1' \in OptAS(P_1)$, then $M_1$ is dominated by $M_1'$ and consequently $M$ is dominated by $M'$. This is a contradiction as $M \in OptAS(P)$ and so $M$ cannot be dominated by any answer set of $P$.
	Thus, the thesis follows.
\end{proof}

The next transformation we propose is used to obtain a clone of an ASP program, so that we can duplicate an ASP program in order to compare the cost of two answer sets of a given program. 
To this end, we need a predicate substitution function that maps each predicate $p$ with a fresh ones of the form $p^{\alpha}_\beta$, where $\alpha$ and $\beta$ are strings.
More in detail, let $L$ be a set of literals and $\epsilon$ be an ASP expression (i.e. rule, program, etc.), then $\sigma^{\alpha}_\beta(L,\epsilon)$ denotes the expression obtained from $\epsilon$ by mapping each positive (resp. negative) literal $p(\vec{t}) \in L$ (resp. $\naf p(\vec{t}) \in L$) with $p^{\alpha}_\beta(\vec{t})$ (resp. $\naf p^{\alpha}_\beta(\vec{t})$).
Let $P$ be an ASP program, then $clone(P)=\sigma_{clone}(\heads{P}\cup\overline{\heads{P}}, P)$, which generates a clone of the program $P$.
Intuitively, a clone of an ASP program $P$ is obtained by mapping all the predicates defined in $P$ (i.e. appearing in some rule head) to a clone signature (i.e. $p_{clone}$) not appearing anywhere else; whereas all the other predicates are kept in their original form as they can be seen as input predicates.
As a result, the clone transformation preserves the answer sets of the program $P$ by mapping atoms in $\heads{P}$ to fresh atoms over the clone signature.
\begin{obs}\label{obs:rename_app}
	Let $P$ be an ASP program, then $M \in AS(P)$ if and only if $\{p_{clone}(\ldots)\mid p(\ldots) \in \heads{P}\cap M\}$. 
\end{obs}
\begin{lemma}\label{lemma:fix_app}
	Let $P_1$ and $P_2$ be two ASP programs such that $\heads{P_2}\cap at(P_1) = \emptyset$, $M_1 \in AS(P_1)$, and $M = M_1\cup M_2$ be an answer set of $P = P_1\cup sat(clone(P_2))$.
	Then, $M' \in AS(P_2\cup\fix{P}{M})$ iff $M'$ is of the form $M_2 \cup M_2'$, with $M_2'\in AS(P_2\cup\fix{P_1}{M_1})$. 
\end{lemma}
\begin{proof}
	Let $P = P_1\cup sat(clone(P_2))$ and $A = \heads{sat(clone(P_2))}$, then $B_P = B_{P_1}\cap A$, with $B_{P_1}\cup A = \emptyset$.
	To this end, for each $M\in AS(P)$, $M = M_1 \cup M_2$ where $M_1 \subseteq B_{P_1}$ and $M_2 \subseteq A$.
	More precisely, from Lemma~\ref{lemma:prg_union_app}, $M_1 \in AS(P_1)$ and $M_2$ is either equal to $\{sat\}$ or $M_2$ is such that $M_1\cup M_2$ is an answer set of $clone(P_2)\cup\fix{P_1}{M_1}$.
	Thus, let $M = M_1\cup M_2 \in AS(P)$, with $M_1 \in AS(P_1)$, then $\fix{P}{M} = \fix{P_1}{M_1}\cup \fix{A}{M_2}$.
	Since atoms in $A$ do not appear neither in $P_2$ nor in $\fix{P_1}{M_1}$, then $A$ is a splitting set for $P_2' = P_2\cup\fix{P}{M} = P_2\cup \fix{P_1}{M_1} \cup \fix{A}{M_2}$. 
	Thus, $M' \in AS(P_2')$ iff $M' = X \cup Y$ where $X$ is an answer set of $b_A(P_2) = \fix{A}{M_2}$ and $Y$ is an answer set of $e_A(t_U(P_2'),X) = e_A(P_2\cup\fix{P_1}{M_1},X)$. 
	Here, $\fix{A}{M_2}$ has only one answer set, that is $M_2$; whereas 
	$e_A(P_2\cup\fix{P_1}{M_1},X) = P_2\cup\fix{P_1}{M_1}$ since none of the atoms in $X$ appears in $P_2\cup\fix{P_1}{M_1}$ and so no rule can be simplified according to $X$.
	Thus, $M' \in AS(P_2')$ iff $M' = M_2 \cup Y$ where $Y \in AS(P_2\cup\fix{P_1}{M})$.
\end{proof}
Finally, the last transformation is used to compute the cost of an answer set as atoms over a fresh predicate $cost/2$.
\setcounter{definition}{4}
\begin{definition}[Cost Program]\label{cost_transformation_app}
    Let $P$ be an ASP program, $L$ be the set of priority levels in $\mathcal{W}(P)$, and $v_P$ and $cl_P$ are fresh predicates not appearing in $P$. Then, $cost(P)$ is defined as:
    \[
    cost(P) = \left\{\begin{array}{lclr}
    v_P(w,l,\vec{t}) & \leftarrow & l_1, \ldots, l_n & \forall  \leftarrow_{\omega} l_1, \ldots, l_n [w@l,\vec{t}] \in P\\
    cl_P(\updated{T},l) &\leftarrow & \#sum\{C,\vec{t}: v_P(C,l,\vec{t})\}=\updated{T}. & \forall l \in L\\
    \end{array}\right.
    \]
\end{definition}
\setcounter{definition}{10}
Basically, Definition~\ref{cost_transformation_app} is a revised version of the $check$ transformation by \citeN{DBLP:journals/tplp/MazzottaRT24}, where each weak constraint is transformed into a rule whose head atoms represent the violation tuples of the weak constraint, and then the each rule with aggregate sums up all the weights assigned to each violation at level $l$.
\begin{example}
	Let $P$ and $W$ be of the form:
	
	\begin{minipage}{.49\textwidth}
		$$
		P = \left\{\begin{array}{l}
			a \leftarrow not\ b  \\
			b \leftarrow not\ a  \\
			c \leftarrow not\ d  \\
			d \leftarrow not\ c  \\
		\end{array}\right\}
		$$
	\end{minipage}
	\hfill
	\begin{minipage}{.49\textwidth}
		$$
		W = \left\{\begin{array}{l}
			\leftarrow^{\omega} a,c\ [1@1]  \\
			\leftarrow^{\omega} a,d\ [2@1]  \\
			\leftarrow^{\omega} b,c\ [1@2]  \\
			\leftarrow^{\omega} b,d\ [2@2]  \\
		\end{array}\right\}
		$$
	\end{minipage}
	
	In this case, the optimization levels are $1$ and $2$ and so $cost(P\cup W)$ is of the form:
	$$
	\begin{array}{l}
		v_P(1,1)\leftarrow a,c\\ 
		v_P(2,1)\leftarrow a,d\\
		v_P(1,2)\leftarrow b,c\\
		v_P(2,2)\leftarrow b,d\\
		cost_P(C,1)\leftarrow \#sum\{1:v_P(1,1),2: v_P(2,1)\}=C\\
		cost_P(C,2)\leftarrow \#sum\{1:v_P(1,2),2: v_P(2,2)\}=C\\
	\end{array}
	$$
\end{example}
As a result, the $cost(P)$ transformation returns a set of rules which serves as post processing to compute the cost of an answer set of $P$ w.r.t. a set of weak constraints $W$. The cost of the answer set at each level $l$ is denoted by auxiliary atoms of the form $cost(c,l)$.

We are now ready to study the complexity of the coherence problem of 2-\ASPQ.
\setcounter{theorem}{0}
\begin{theorem}[Membership]\label{lemma:membership_app}
	The coherence problem for \TWOASPQW programs of the form $\Box_1 P_1\Box_2 P_2:C$ is in: $\Sigma_2^P$ if $\Box_1 = \exists^{st}$; $\Pi_2^P$ otherwise,
    \updated{
    no matter whether $\Box_2 = \exists^{st}$ or $\Box_2 = \forall^{st}$.
    } 
\end{theorem}
\setcounter{theorem}{7}
\begin{proof}
    For the sake of the presentation, in what follows we keep separate weak constraints from rules of ASP subprogram. 
    Thus, we consider \TWOASPQW programs of the form $\Box_1 P_1\cup W_1\Box_2 P_2\cup W_2:C$, where $P_1$ and $P_2$ are ASP programs without weak constraints; whereas $W_1$ and $W_2$ are set of weak constraints.
	To prove our thesis we need to distinguish three different scenarios.
	
	Uniform quantifiers $\Box_1 = \Box_2$. By applying the transformation (Algorithm 1) by \citeN{DBLP:journals/tplp/MazzottaRT24}, it is possible to obtain an alternating plain 2-\ASPQ program $\Pi'$ of the form $\Box_1 P_1'\overline{\Box_1} P_2':C'$ such that $\Pi'$ is coherent if and only if $\Pi$ is coherent. 
	From the result of~\citeN{DBLP:journals/tplp/AmendolaRT19}, verifying the coherence of $\Pi'$ is in $\Sigma_2^P$ if $\Box_1 = \exists^{st}$; otherwise is in $\Pi_2^P$. Thus the thesis holds for uniform quantifiers.
	
	Case $\Box_1 = \exists^{st}$ and $\Box_2 = \forall^{st}$. 
	In this case, $\Pi$ can be encoded into a plain 2-\ASPQ program. 
	Such translation requires two steps: first of all we remove weak constraints in $W_2$, thus obtaining a 2-\ASPQ program $\Pi'$ where weak constraint occurs only in the first program. Then, by applying the transformation (Algorithm 1) by \citeN{DBLP:journals/tplp/MazzottaRT24}, we get rid of weak constraint in $W_1$ and obtain a plain 2-\ASPQ program that is coherent iff and only if $\Pi$ is coherent.
	
	Let us now focus on the first step. To remove weak constraint in $W_2$ while preserving the coherence of $\Pi$ we can construct a 2-\ASPQ $\Pi'$ of the form $$ \exists^{st} P_1 \cup W_1 \cup sat(clone(P_2))\forall^{st} P_2: C',$$
	where $C' = cost(clone(P_2\cup W_2)) \cup cost(P_2\cup W_2) \cup C^{opt}$, and $C^{opt}$ is of the form:
	$$
	\begin{array}{lr}
		diff(L)\leftarrow cost(C1,L), const_{clone}(C2,L), C1\neq C2      & \\
		hasHigher(L1) \leftarrow diff(L1),diff(L2), L2>L1           & \\
		highest(L) \leftarrow diff(L),not\ hasHigher(L)           & \\
		dom_{clone} \leftarrow highest(L), const_{clone}(C1,L),cost(C2,L), C1>C2            & \\
		diffCost \leftarrow diff(L)                           & \\
		equalCost \leftarrow not\ diffCost                         & \\
		H_r \leftarrow B_r, equalCost                              & \forall r \in C \\
		\leftarrow dom_{clone}                                           & \\
		\leftarrow sat                                             & \\
	\end{array}
	$$
	Intuitively, $P_1\cup W_1\cup sat(clone(P_2))$ augments $P_1\cup W_1$ with a clone of $P_2$ which always admits an answer set, the program $P_2$ is kept in its original form, and, finally, the program $C'$ computes the cost of answer sets of $sat(clone(P_2))$ and $P_2$, respectively. Moreover, the program $C'$ contains also the rules which:
	\begin{itemize}
		\item compare the cost of the two answer sets and derive an atom of the form $diff(l)$ iff the cost of the compared answer sets differs at level $l$, and as consequence the atom $diffCost$ is derived as well;
		\item derive an atom of the form $highest(l)$ if and only if $l$ is the highest level at which the cost of the two answer sets differs;
		\item derive an atom $dom_{clone}$ iff the answer set of $clone(P_2)$ is dominated by the answer set of $P_2$;
		\item derive an atom $equalCost$ if and only if there is no level at which the cost of the two answer sets differs, and thus the two answer sets have the same cost.
	\end{itemize}
	Finally, the atom $equalCost$ is used to control the activation of the rules from the original constraint program $C$. More precisely, if $equalCost$ is derived as false, then all the rules of $C$ are trivially satisfied. 
	On the contrary, if $equalCost$ is derived as true, then all the rules of $C$ are obtained back as $equalCost$ can be removed from all rules bodies. 
	We now prove that $\Pi$ is coherent if and only if $\Pi'$ is coherent.
	
	\noindent($\Rightarrow$) We assume $\Pi$ to be coherent and we prove that $\Pi'$ is coherent as well.
	Since $\Pi$ is coherent, then there exists $M_1\in OptAS(P_1\cup W_1)$ such that $\forall^{st} P_2 \cup W_2\cup\fix{P_1}{M_1}:C$ is coherent.
	
	This means that either $P_2\cup\fix{P_1}{M_1}$ is incoherent (if there are no answer sets then the forall check is trivially satisfied) or there is no optimal answer set of $P_2\cup W_2 \cup \fix{P_1}{M_1}$ that violates the program $C$. 
	In both cases, we will show that it is possible construct a quantified answer set of $\Pi'$, and so $\Pi'$ is coherent as well.
	We recall that $\Pi'$ is of the form $\exists^{st} P_1\cup W_1\cup sat(clone(P_2))\forall^{st} P_2 : C'$. 
	
	\begin{itemize}
		\item $P_2\cup\fix{P_1}{M_1}$ is incoherent.
		
		In this case, we show that $M = M_1\cup\{sat\}$ is a quantified answer set of $\Pi'$. 
		
		From Lemma~\ref{lemma:prg_union_weak_app}, $M$ is an optimal answer set of $P = P_1\cup W_1\cup sat(clone(P_2))$, and so, to prove that $M$ is also a quantified answer set of $\Pi'$, we show that $\forall^{st} P_2\cup\fix{P}{M}:C$ is coherent.
		From Lemma~\ref{lemma:fix_app}, $M' \in AS(P_2\cup\fix{P}{M})$ iff $M'$ is of the form $\{sat\}\cup M_2'$, with $M_2' \in AS(P_2\cup\fix{P_1}{M_1})$.
		Since $P_2\cup\fix{P_1}{M_1}$ is incoherent, then $AS(P_2\cup\fix{P_1}{M_1}) = \emptyset$ and so $P_2\cup\fix{P}{M}$ is incoherent as well. Thus, $\forall^{st} P_2\cup\fix{P}{M}:C$ is coherent and $M$ is a quantified answer set of $\Pi'$.
		
		\item No optimal answer set of $P_2\cup W_2 \cup \fix{P_1}{M_1}$ violates $C$.
		
		Let $M_2^c \subseteq \heads{sat(clone(P_2))}$ be such that $M_1 \cup M_2^c \in OptAS(clone(P_2\cup W_2)\cup\fix{P_1}{M_1})$, then we are going to show that $M = M_1\cup M_2^c$ is a quantified answer set of $\Pi'$.
		
		Let $P = P_1\cup W_1\cup sat(clone(P_2))$, then, from Lemma~\ref{lemma:prg_union_app}, $M\in OptAS(P)$ and so $M$ is a quantified answer set of $\Pi'$ iff $\forall^{st} P_2\cup\fix{P}{M}:C'$ is coherent.
		From Lemma~\ref{lemma:fix_app}, $M'$ is an answer set of $P_2\cup\fix{P}{M}$ iff $M'$ is of the form $M_2^c \cup M_2'$, where $M_2' \in AS(P_2\cup\fix{P_1}{M_1})$.
		Thus, let $M' = M_2^c\cup M_2' \in AS(P_2\cup\fix{P}{M})$, we show that $M'$ satisfies $C'$.
		
		First of all, we observe that the constraint $\leftarrow sat \in C'$ is trivially satisfied as $sat$ is false w.r.t. $M$.
		Moreover, we observe that since $M$ is optimal w.r.t. weak constraint in $W_2$, then $M$ is not dominated by $M'$. This means that $dom_{clone}$ is derived as false and the constraint $\leftarrow dom_{clone}$ is satisfied.
		Finally, we need to prove that all the rules of the form $H_r\leftarrow B_r,equalCost$ in $C'$, with $r \in C$, are satisfied w.r.t. $M'$.
		To this end, we need to distinguish two cases depending on whether $M_2'$ is optimal or not w.r.t. weak constraint $W_2$.
		If $M_2'$ is not optimal then $M'$ is dominated by $M$ and so there exists a level $l$ such that the cost of $M'$ is greater than the cost of $M$. Thus, the cost of $M'$ is not equal to the cost of $M$ and so the atom $equalCost$ is derived as false. 
		Since $equalCost$ appears positively in all the remaining rules, then they are satisfied w.r.t. $M'$.
		On the other hand, when $M_2$ is optimal, then $M'$ and $M$ have the same cost for each level $l$ and so the atom $equalCost$ is derived as true. This means that for each $r \in C$, $H_r\leftarrow B_r,equalCost \in C'$ is satisfied w.r.t. $M'$ iff $H_r\leftarrow B_r$ is satisfied w.r.t. $M'$ which means the constraint program $C$ is satisfied w.r.t. $M'$. 
		Since none of the atoms in $\heads{sat(clone(P_2))}$ appears in $C$, then $C$ is satisfied w.r.t. $M' = M_2^c \cup M_2'$ iff $C$ is satisfied w.r.t. $M_2'$. By assumption, no optimal answer set of $P_2\cup W_2\cup\fix{P_1}{M_1}$ violates $C$ and so $C$ is satisfied w.r.t. $M_2'$ and $C'$ is satisfied also w.r.t. $M'$.
		Here we observe that $M'$ was chosen arbitrarily among answer set of $P_2\cup\fix{P}{M}$ and $M'$ satisfies $C'$. Thus, this means that $C'$ is satisfied w.r.t. each $M' \in AS(P_2\cup\fix{P}{M})$ and so $M$ is a quantified answer set of $\Pi'$. 
	\end{itemize}
	\noindent($\Leftarrow$) We assume $\Pi'$ to be coherent and we prove that $\Pi$ is coherent as well.
	Let $P = P_1\cup W_1\cup sat(clone(P_2))$, then there exists $M \in AS(P)$ such that $\forall^{st} P_2\cup\fix{P}{M}:C'$ is coherent.
	From Lemma~\ref{lemma:prg_union_weak_app}, $M$ is of the form $M_1\cup M_2^c$ where $M_1 \in OptAS(P_1\cup W_1)$ and $M_2^c$ is either equal to $\{sat\}$ or $M_2^c$ is such that $M_1\cup M_2^c \in AS(clone(P_2)\cup\fix{P_1}{M_1})$.
	Thus, to prove the coherence of $\Pi$, we are going to show that $M_1$ is quantified answer set of $\Pi$.  
	
	\begin{itemize}
		\item $M_2^c = \{sat\}$.
        
		Since $\forall^{st} P_2\cup\fix{P}{M}: C'$ is coherent and $C'$ contains the constraint $\leftarrow sat$, then $P_2\cup\fix{P}{M}$ is incoherent.
		From Lemma~\ref{lemma:fix_app}, we know that $M' \in AS(P_2\cup\fix{P}{M})$ if and only if $M'$ is of the form $\{sat\} \cup M_2'$, with $M_2' \in AS(P_2\cup\fix{P_1}{M_1})$.
		Since $P_2\cup\fix{P}{M}$ is incoherent, then $P_2\cup W_2\cup\fix{P_1}{M_1}$ is incoherent as well, and thus, $\forall^{st} P_2\cup W_2\cup\fix{P_1}{M_1}:C$ is coherent and $M_1$ is a quantified answer set of $\Pi$.
		
		\item $M = M_1 \cup M_2^c \in AS(clone(P_2)\cup\fix{P_1}{M_1})$.
		
		Let $M_2 \in OptAS(P_2\cup W_2 \cup \fix{P_1}{M_1})$.
		From Lemma~\ref{lemma:fix_app}, we know that $M' = M_2^c \cup M_2 \in AS(P_2\cup\fix{P}{M})$, and so, by assumption, $M'$ satisfies $C'$. So this means the each rule of $C'$ is satisfied. 
		More precisely, the constraint $\leftarrow sat$ is satisfied as the atom $sat$ is false w.r.t. $M$; to satisfy the constraint $\leftarrow dom_{clone}$ then $M$ cannot be dominated by $M'$ and so $M$ is optimal w.r.t. weak constraints in $clone(P_2\cup W_2)$.
		At this point, the last rules to satisfy from $C'$ are of the form $H_r\leftarrow B_r, equalCost$, with $r \in C$.
		Since both $M$ and $M'$ are optimal, then the cost of $M$ is equal to the cost of $M'$ for each level $l$ and thus the atom $equalCost$ is derived as true.
		This means that, $H_r\leftarrow B_r, equalCost$ is satisfied w.r.t. $M'$ iff $H_r\leftarrow B_r$ is satisfied w.r.t. $M'$. Since none of the atoms in $\heads{sat(clone(P_2))}$ occurs in $C$, then $H_r\leftarrow B_r$ is satisfied w.r.t. $M' = M_2^c\cup M_2$ iff $H_r\leftarrow B_r$ is satisfied w.r.t. $M_2$.
		So this means the the constraint program $C$ is satisfied w.r.t. $M_2$.
		
		Since $M_2$ was chosen arbitrarily among optimal answer sets of $P_2\cup W_2\cup \fix{P_1}{M_1}$, it follows that every optimal answer set of $P_2\cup\fix{P_1}{M_1}$ satisfies $C$, and so $M_1$ is a quantified answer set of $\Pi$.
	\end{itemize}
	
	Thus, we can conclude that $\Pi$ is coherent iff $\Pi'$ is coherent. At this point, we observe that $\Pi'$ contains weak constraints only in the first program. Thus, by applying the transformations by~\citeN{DBLP:journals/tplp/MazzottaRT24} (i.e. $col_5(\cdot)$) it is possible to obtain a plain 2-\ASPQ program $\Pi''$ of the form $\exists^{st} P_1'' \forall^{st} P_2'': C''$ which is coherent iff $\Pi'$ is coherent. Since verifying the coherence of $\Pi''$ is a $\Sigma_2^P$-complete problem~\cite{DBLP:journals/tplp/AmendolaRT19}, then the thesis holds in this case.
    
	Case $\Box_1 = \forall^{st}$ and $\Box_2 = \exists^{st}$. 
	As in the previous case, $\Pi$ can be encoded into a plain 2-\ASPQ program. Again, the translation requires two steps: first of all we remove weak constraints in $W_2$, thus obtaining a 2-\ASPQ program $\Pi'$ where weak constraints occur only in the first program. Then, by applying the transformations by \citeN{DBLP:journals/tplp/MazzottaRT24}, we get rid of weak constraints in $W_1$ and obtain a plain 2-\ASPQ program which is coherent iff and only if $\Pi$ is coherent.
	
	More precisely, the 2-\ASPQ program $\Pi$ can be encoded into a $2$-\ASPQ program $\Pi'$ of the form $\forall^{st} P_1 \cup sat(clone(P_2))\exists^{st} P_2: C'$, where $C' = cost(clone(P_2\cup W_2)) \cup cost(P_2\cup W_2) \cup C^{opt}$, and $C^{opt}$ is of the form:
	$$
	\begin{array}{lr}
		diff(L)\leftarrow cost(C1,L), const_{clone}(C2,L), C1\neq C2      & \\
		hasHigher(L1) \leftarrow diff(L1),diff(L2), L2>L1           & \\
		highest(L) \leftarrow diff(L),not\ hasHigher(L)           & \\
		dom_{clone} \leftarrow const_{clone}(C1,L),cost(C2,L), C1>C2            & \\
		diffCost \leftarrow highest(L)                           & \\
		equalCost \leftarrow not\ diffCost                         & \\
		\leftarrow not\ dom_{clone}, not\ equalCost, not\ sat                                           & \\
		
		H_r \leftarrow B_r, not\ dom_{clone}, not\ sat                              & \forall r \in C \\
	\end{array}
	$$
	More precisely, $\Pi$ is incoherent iff $\Pi'$ is incoherent.
	
	\noindent($\Rightarrow$) Let us assume $\Pi$ to be incoherent, then we show that $\Pi'$ is incoherent as well. Since $\Pi$ is incoherent, then there exists $M_1 \in OptAS(P_1\cup W_1)$ such that $\exists^{st} P_2\cup W_2\cup \fix{P_1}{M_1}:C$ is incoherent. Thus, to prove the incoherence of $\Pi'$, we construct a witness of its incoherence, that is $M \in OptAS(P)$, with $P = P_1\cup W_1\cup sat(clone(P_2))$, such that $\exists^{st}P_2\cup\fix{P}{M}:C'$ is incoherent.
	
	More in detail, since $\exists^{st} P_2\cup W_2\cup \fix{P_1}{M_1}: C$ is incoherent, then either $P_2\cup\fix{P_1}{M_1}$ is incoherent or there is no optimal answer set of $P_2\cup W_2\cup\fix{P_1}{M_1}$ that satisfies $C$. Let us distinguish the two cases separately.
	\begin{itemize}
		\item $P_2 \cup \fix{P_1}{M_1}$ is incoherent.
		
		In this case, we construct an answer set $M = M_1 \cup \{sat\}$ that serve as witness for the incoherence of $\Pi'$.
		Specifically, from Lemma~\ref{lemma:prg_union_weak_app}, $M = M_1 \cup \{sat\} \in OptAS(P)$.
		Moreover, from Lemma~\ref{lemma:fix_app}, $M' \in AS(P_2\cup\fix{P}{M})$ if and only if $M'$ is of the form $\{sat\}\cup M_2'$, with $M_2'\in AS(P_2\cup\fix{P_1}{M_1})$. 
		Since $P_2\cup\fix{P_1}{M_1}$ is incoherent, then $AS(P_2\cup\fix{P_1}{M_1}) = \emptyset$, and $AS(P_2\cup\fix{P}{M})= \emptyset$. As a result, $P_2\cup\fix{P_1}{M_1}$ is incoherent and $M$ is a witness for the incoherence of $\Pi'$.
		
		\item There is no optimal answer set of $P_2\cup W_2\cup\fix{P_1}{M_1}$ that satisfies $C$.
		
		Let $M_2^c \subseteq \heads{sat(clone(P_2))}$ be such that $M_1\cup M_2^c \in OptAS(clone(P_2\cup W_2)\cup\fix{P_1}{M_1})$.
		Then, from Lemma~\ref{lemma:prg_union_weak_app}, $M = M_1\cup M_2^c \in OptAS(P)$, thus we need to prove that $M$ is a witness for the incoherence of $\Pi'$. 
		To this end, we show that there is no $M' \in AS(P_2\cup\fix{P}{M})$ that satisfies $C'$.
		
		More precisely, from Lemma~\ref{lemma:fix_app}, $M' \in AS(P_2\cup\fix{P}{M})$ iff $M'$ is of the form $M_2^c \cup M_2'$, with $M_2' \in AS(P_2\cup\fix{P_1}{M_1})$.
		Thus, let $M' = M_2^c\cup M_2' \in AS(P_2\cup\fix{P}{M})$, with $M_2'\in AS(P_2\cup\fix{P_1}{M_1})$, we show that $M'$ violates $C'$.
		To this end, we distinguish two possible cases depending on whether $M_2'$ is optimal w.r.t. weak constraints in $W_2$ or not.
		
		If $M_2'$ is optimal, then the cost of $M$ is equal to the cost of $M'$ for each level $l$, and so the atom $equalCost$ is derived as true. Since $equalCost$ is true, then the constraint $\leftarrow not\ dom_{clone},\ not\ equalCost,\ not\ sat$ is satisfied.
		Moreover, since the atom $sat$ is false w.r.t. $M$ and the atom $dom_{clone}$ is derived as false as $M$ is not dominated by $M'$, then each rule $H_r\leftarrow B_r,not\ dom_{clone},not\ sat \in C'$, with $r \in C$, is satisfied w.r.t. $M'$ iff $H_r \leftarrow B_r$ is satisfied w.r.t. $M'$. Since none of the atoms in $\heads{sat(clone(P_2))}$ appears in $C$, then $H_r \leftarrow B_r$ is satisfied w.r.t. $M' = M_2^c\cup M_2'$ iff $H_r \leftarrow B_r$ is satisfied w.r.t. $M_2'$ which means $M_2'$ satisfies the constraint program $C$. By assumption, $M_2'$ does not satisfy $C$ and so $C'$ is not satisfied w.r.t. $M'$.
		
		Conversely, if $M_2'$ is not optimal then $M'$ is dominated by $M$ and so, there exists at least on level $l$ such that the cost of $M'$ at level $l$ is greater than the cost of $M$ at level $l$. As a result, the atom $equalCost$ is derived as false.
		Moreover, both atoms $sat$ and $dom_{clone}$ are derived as false since $sat$ is false w.r.t. $M$ and $M$ is optimal w.r.t. weak constraints in $clone(P_2\cup W_2)$ and so $M$ cannot be dominated by $M'$.
		As a result, the constraint $\leftarrow not\ dom_{clone},\ not\ equalCost,\ not\ sat$ is violated and so $C'$ is not satisfied w.r.t. $M'$.
		
		Since in both cases $M'$ violates $C'$ and $M'$ can be any answer set of $P_2\cup\fix{P}{M}$, then no answer set of $P_2\cup\fix{P}{M}$ satisfies $C'$ and thus $M = M_1\cup M_2^c$ is a witness for the incoherence of $\Pi'$.
	\end{itemize}
	
	\noindent ($\Leftarrow$) Let us assume that $\Pi'$ is incoherent, then we show that $\Pi$ is incoherent as well. More precisely, since $\Pi'$ is incoherent, then there exists $M \in OptAS(P)$, with $P = P_1\cup W_1\cup sat(clone(P_2))$, such that $\exists^{st} P_2\cup\fix{P}{M}:C'$ is incoherent.
	
	From Lemma~\ref{lemma:prg_union_weak_app}, $M$ is of the form $M_1\cup M_2^c$, where $M_1 \in OptAS(P_1\cup W_1)$ and $M_2^c\subseteq \heads{sat(clone(P_2))}$ is either equal to $\{sat\}$ or it is such that $M_1 \cup M_2^c \in AS(clone(P_2)\cup\fix{P_1}{M_1})$.
	Thus, to prove the incoherence of $\Pi$, we are going to show that $M_1$ is a witness for the incoherence of $\Pi$ in both cases.
    
	\begin{itemize}
		\item ($M = M_1\cup \{sat\}$)
        
        Since the atom $sat$ is true w.r.t. $M$ and $sat$ appears negatively in each constraint in $C'$, then the only way for making $\exists^{st} P_2\cup\fix{P}{M}:C'$ be incoherent is that $P_2\cup\fix{P}{M}$ is incoherent.
		From Lemma~\ref{lemma:fix_app}, $M' \in AS(P_2\cup\fix{P}{M})$ iff $M'$ is of the form $\{sat\}\cup M_2'$, with $M_2'\in AS(P_2\cup\fix{P_1}{M_1})$. Since $P_2\cup\fix{P}{M}$ is incoherent, then $AS(P_2\cup\fix{P}{M}) = \emptyset$ and so, also $AS(P_2\cup\fix{P_1}{M_1}) = \emptyset$. Thus, $P_2\cup\fix{P_1}{M_1}$ is incoherent and so $M_1$ is a witness for the incoherence of $\Pi$.
		
		\item ($M = M_1\cup M_2^c \in AS(clone(P_2)\cup\fix{P_1}{M_1})$)

		Here we observe that the atom $dom_{clone}$ appears negatively in each constraint of $C'$. 
		This means that if $M$ is not optimal w.r.t. weak constraint in $clone(P_2\cup W_2)$, then $M$ would be dominated by some answer set of $P_2\cup\fix{P}{M}$ and the program $C'$ would be satisfied. 
		Since by assumption $\exists^{st} P_2\cup\fix{P}{M}:C'$ is incoherent, then $M$ is optimal w.r.t. weak constraints in $clone(P_2\cup W_2)$. 
		Moreover, we observe that the atom $sat$ is false w.r.t. $M$ and so it cannot trivially satisfy the constraints in $C'$.
		
		Let $M_2 \in OptAS(P_2\cup W_2\cup \fix{P_1}{M_1})$, then, from Lemma~\ref{lemma:fix_app}, $M' = M_2^c\cup M_2$ is an answer set of $P_2\cup\fix{P}{M}$. Since $M_2$ is optimal, then the cost of $M$ is equal to the cost of $M'$ for each level $l$ and so the atom $equalCost$ is derived as true and thus the constraint $\leftarrow not\ dom_{clone},not\ equalCost,not\ sat \in C'$ is satisfied.
		By assumption, $M'$ violates $C'$ which means that there exists at least one rule $H_r\leftarrow B_r,not\ dom_{clone},not\ sat \in C'$, with $r \in C$, that is violated w.r.t. $M'$. Since $dom_{clone}$ and $sat$ are derived as false and none of the atoms in $\heads{sat(clone(P_2))}$ appears in $H_r\leftarrow B_r \in C$, then $H_r\leftarrow B_r$ is violated w.r.t. $M' = M_2^c\cup M_2$ iff $H_r\leftarrow B_r$ is violated w.r.t. $M_2$.
		Thus, $M_2$ violates the original constraint program $C$.
		
		Since $M_2$ can be any optimal answer set of $P_2\cup W_2\cup\fix{P_1}{M_1}$ and $M_2$ violates $C$, then no optimal answer set of $P_2\cup W_2\cup\fix{P_1}{M_1}$ satisfies $C$.
		Thus, $\exists^{st} P_2\cup W_2\cup \fix{P_1}{M_1}:C$ is incoherent and so $M_1$ is a witness for the incoherence of $\Pi$. 
	\end{itemize}
	At this point we have showed that $\Pi'$ preserves the (in)coherence of $\Pi$. As in the previous case, we observe that $\Pi'$ contains only weak constraints (i.e., $W_1$) in the first program and so, to get rid of such weak constraints it is possible to use the transformation by~\citeN{DBLP:journals/tplp/MazzottaRT24} (i.e. $col_6(\cdot)$) which returns a plain 2-\ASPQ $\Pi''$ of the form $\forall^{st} P_1''\exists P_2'':C''$ that is cohrent iff $\Pi'$ is coherent. Since verifying the coherence of $\Pi''$ is a $\Pi_2^P$-complete problem~\cite{DBLP:journals/tplp/AmendolaRT19}, then the thesis holds also in this last case.
\end{proof}
\setcounter{theorem}{1}
\begin{theorem}[Hardness]\label{lemma:hardness_app}
	The coherence problem for \TWOASPQW programs of the form $\Box_1 P_1\Box_2 P_2:C$ is hard for: 
    $\Sigma_2^P$ if $\Box_1 = \exists^{st}$; $\Pi_2^P$ otherwise,
    \updated{
    no matter whether $\Box_2 = \exists^{st}$ or $\Box_2 = \forall^{st}$.
    } 
\end{theorem}
\setcounter{theorem}{7}
\begin{proof}
    For the sake of the presentation, in what follows we keep separate weak constraints from rules of ASP subprogram. 
    Thus, we consider \TWOASPQW programs of the form $\Box_1 P_1\cup W_1\Box_2 P_2\cup W_2:C$, where $P_1$ and $P_2$ are ASP programs without weak constraints; whereas $W_1$ and $W_2$ are set of weak constraints.
    
	As in the previous case, to prove our thesis we need to distinguish different scenarios according to possible pairs of quantifiers.
	
	Case $\Box_1 = \Box_2 = \exists^{st}$. In this case, \citeN{DBLP:journals/tplp/MazzottaRT24} proved that the coherence problem for $\Pi$ is $\Sigma_2^P$-complete if $W_1$ is empty. Since this is a particular case, then the thesis holds also for the general one.
	
	Case $\Box_1 = \Box_2 = \forall^{st}$. In this case, we prove the hardness for the particular case in which $W_1$ is empty which inherently gives us the hardness for the general one.
	
	Let $\Phi = \forall X\exists Y \phi$ be a $2$-QBF, where $X$ and $Y$ are disjoint set of variables and $\phi$ is a boolean formula in 3-CNF. 
	We recall that a formula in 3-CNF is of the form $c_1 \wedge \ldots\wedge c_n$, where each $c_i$, with $1\leq i\leq n$, is a disjunction of the form $l_1^i\vee l_2^i\vee l_3^i$ and each $l_j^i$ is either a positive (i.e., $a \in X\cup Y$) or negative (i.e., $\neg a$, with $a \in X\cup Y$) literal.
	Given the 2-QBF $\Phi$, verifying that $\Phi$ is true is a $\Pi_2^P$-complete problem~\cite{schaefer2002completeness}. 
	Such a problem can be encoded as a 2-\ASPQ program $\Pi$ of the form $\forall^{st} P_1 \forall^{st} P_2\cup W_2:C$ that is coherent if and only if $\Phi$ is true.
	More precisely, the program $P_1$ is made of rules of the form $tau'(x,t)\leftarrow not\ tau'(x,f)$ and $tau'(x,f)\leftarrow not\ tau'(x,t)$ for each $x \in X$.
	The program $P_2$ is of the form
	$$
	\begin{array}{lr}
		tau(x,t)\leftarrow tau'(x,f) &  \forall x \in X\\
		tau(x,t)\leftarrow tau'(x,t) &  \forall x \in X\\
		tau(y,t)\leftarrow not\ tau(y,f) &  \forall y \in Y\\
		tau(y,f)\leftarrow not\ tau(y,t) &  \forall y \in Y\\
		sat(c_i)\leftarrow tau(a,f)         & \forall c_i \in \phi, l_j^i=\neg a \in c_i\\
		sat(c_i)\leftarrow tau(a,t)         & \forall c_i \in \phi, l_j^i=x \in c_i, a \in a\cup Y\\
		unsat \leftarrow not\ sat(c_i)      & \forall c_i \in \phi
	\end{array}
	$$
	The set of weak constraints $W_2 = \{\leftarrow^{\omega} unsat\ \ [1]\}$, and the constraint program $C$ contains only the constraint $\leftarrow unsat$.
	In this case, the answer set of the program $P_1$ are in one-to-one correspondence with possible truth assignment $\tau'$ for the variables in $X$. Thus, there exists $M_1 \in AS(P_1)$ if and only if there exists a truth assignment $\tau'$ for variables in $X$.
	Then, the program $P_2' = P_2 \cup W_2\cup \fix{P_1}{M_1}$ guesses a truth assignment $\tau$ for the variables in $Y$ by extending the assignment $\tau'$ encoded by $M_1$.
	Moreover, the program $P_2$ checks whether $\tau$ satisfies $\phi$. To this end, for each clause $c_i \in \phi$, $sat(c_i)$ is derived iff there exists a positive (resp. negative) literal $l_i^j = a$ (resp. $l_i^j = \neg a$), with $a \in X\cup Y$, such that $\tau(a)$ is true (resp. $\tau(a)$ is false).
	So, $sat(c_i)$ is derived as true if and only if the clause $c_i$ in $\phi$ is satisfies w.r.t. $\tau$.
	Then, the atom $unsat$ is derived if and only if there is at least one clause that is not satisfied w.r.t. $\tau$ that is $\phi$ is false w.r.t. $\tau$.
	Thus, let $M_2 \in AS(P_2')$ then $unsat \in M_2$ iff there exists $\tau$ which extends $\tau'$ and $\phi$ is false w.r.t. $\tau$.
	In this case, the weak constraint $\leftarrow^{\omega} unsat\ \ [1]$ add a penalty to $M_2$. Thus, answer sets of $P_2'$ that do not contain $unsat$ are preferred as they do not add any penalty.
	So, let $M_2 \in OptAS(P_2')$ if $unsat \in M_2$, then the program $C \cup\fix{P_2'}{M_2}$ is incoherent and so also the program $\Pi$ is incoherent. 
	At the same time, we observe that for every $M_2'\in AS(P_2')$, $unsat \in M_2'$ because otherwise $M_2$ would not be optimal. Thus, for each assignment $\tau$ of variables in $Y$ which extends $\tau'$, $\phi$ is false w.r.t. $\tau$ and so for the truth assignment $\tau'$ of variables in $X$, it does not exists an assignment $\tau$ of variables in $Y$ that make $\phi$ true. 
	Thus $\Phi$ is not true.  
	On the other hand, if $unsat\notin M_2$ then for each $M_2' \in OptAS(P_2')$, $unsat \notin M_2'$ because otherwise $M_2'$ would be dominated by $M_2$ and it would not be optimal. 
	Thus, for each $M_2 \in OptAS(P_2')$, $C\cup\fix{P_2'}{M_2}$ is coherent and so $M_1$ is not a witness for the incoherence of $\Pi$.
	At the same, since $unsat \notin M_2$ then $\phi$ is true w.r.t. $\tau$ and so for the assignment $\tau'$ of variables in $X$ there exists the assignment $\tau$ of variables in $Y$ which extends $\tau'$ and make $\phi$ true, and so $\tau$ is not a witness for $\Phi$ not being true.
	
	Thus, we can conclude that $\Pi$ is incoherent iff $\Phi$ is not true, and so the thesis follows. 
	
	Case $\Box_1 \neq \Box_2$. Here the results follows from the complexity of the coherence problem for plain 2-\ASPQ programs. Indeed, a plain 2-\ASPQ of the form $\Box_1 P_1 \Box_2 P_2:C$ is a 2-\ASPQ program of the form $\Box_1 P_1\cup W_1\Box_2 P_2\cup W_2:C$, where $W_1$ and $W_2$ are empty. 
	\citeN{DBLP:journals/tplp/AmendolaRT19} proved that verifying the coherence of $\Box_1 P_1 \Box_2 P_2:C$ is $\Sigma_2^P$-complete if $\Box_1 = \exists^{st}$; otherwise it is $\Pi_2^P$-complete. Thus, the thesis follows.
\end{proof}
\setcounter{theorem}{2}
\begin{theorem}[Completeness]\label{thm:complete_two_aspq_app}
	The coherence problem for \TWOASPQW programs of the form $\Box_1 P_1\Box_2 P_2:C$ is: $\Sigma_2^P$-complete if $\Box_1 = \exists^{st}$; $\Pi_2^P$-complete otherwise,
    \updated{
    no matter whether $\Box_2 = \exists^{st}$ or $\Box_2 = \forall^{st}$.
    }
\end{theorem}
\begin{proof}
	The thesis follows from Lemma~\ref{lemma:membership_app} and~\ref{lemma:hardness_app}.
\end{proof}
Theorem~\ref{thm:complete_two_aspq_app} gives as an important results which allows to give completeness results for the brave reasoning tasks in \TWOASPQW, which consists of verifying whether an atom $a$ is true in at least one optimal quantified answer set.
\begin{theorem}
    Let $\Pi$ be an existential \TWOASPQW program of the form $\Box_1 P_1\Box_2 P_2:C:C^{\omega}$ and $a$ be a ground atom, then verifying whether $a \in M$, with $M \in OptQAS(\Pi)$, is $\Delta_{3}^P$-complete \updated{
    no matter whether $\Box_2 = \exists^{st}$ or $\Box_2 = \forall^{st}$.
    }
\end{theorem}
\begin{proof}
    (Hardness) If we consider the particular case where $\mathcal{W}(P_1) = \mathcal{W}(P_2) = \emptyset$ then verifying whether there exists $M \in OptQAS(\Pi)$ such that $a \in M$ is $\Delta_{3}^P$-complete~\cite{DBLP:journals/tplp/MazzottaRT24}. Thus, it holds also for the general case.

    \noindent(Membership)
    As it has been observed by~\citeN{DBLP:journals/tplp/MazzottaRT24}
    an optimal quantified answer set of $\Pi$ can be obtained with binary search on the value of maximum possible cost, namely $k$.
    Since $k$ can be exponential in the general case, then an optimal quantified answer set of $\Pi$ can be obtained with a polynomial number of calls to the oracle in $\Sigma_2^P$ which decides the coherence of a \TWOASPQW program. 
    Finally, an extra oracle call checks that the atom $a$ appears in some optimal quantified answer sets. 
\end{proof}
\setcounter{theorem}{7}
\section{Examples: Program Transformation}\label{sec:program_transform}
In this section, we redefine all the transformations required by our CEGAR-based approach for solving \TWOASPQW. We then illustrate each transformation through examples, with particular emphasis on \emph{counterexample search} and \emph{abstraction refinement} stages.

We recall that the \TWOASPQW programs we consider in our approach are of the form:
\begin{equation}\label{eq:2_aspqw_app}
    \Box P_1\overline{\Box} P_2:C:C^{\omega}
\end{equation}   

\subsection{Counterexample Search}\label{sec:program_transform_ctr}
\setcounter{definition}{0}
\begin{definition}[Complement of stratified program]\label{def:complement_transformation_app}
Let $P$ be a stratified ASP program with hard constraints, then the \emph{complement} of $P$, denoted by $\neg P$, is obtained from $P$ by $(i)$ transforming hard constraints $\leftarrow l_1,\ldots,l_n$ into rules of the form $v \leftarrow l_1,\ldots,l_n$; and $(ii)$ adding an hard constraint of the form $\leftarrow\naf v$.
\end{definition}
\setcounter{definition}{10}

The complement of a stratified program $P$ with constraints is obtained by ($i$) transforming each hard constraint into a rule having as head the fresh atom $v$ and by ($ii$) adding an hard constraint which forces the fresh atom $v$ to be true. As a result, at least one hard constraint from $P$ must be violated for $v$ to be derived as true and thus for $\neg P$ to be incoherent.
The following example clarifies how the complement of a stratified program works. 
\begin{example}\label{complement_example_app}
    Consider the following stratified program $C$ and its complement $\neg C$ where $v \in B_{\neg C}$ is the fresh symbol introduced to capture constraint violations from $C$.
    
    \begin{minipage}{0.20\textwidth}
        \[
            C = \left\{\begin{array}{lcr}
            l& \leftarrow & d,\ c.\\
            & \leftarrow & l,\ b\\
              \end{array}\right\}
        \]        
    \end{minipage}
    \hspace{0.16\textwidth}
    \begin{minipage}{0.20\textwidth}
        \[
            \neg C = \left\{\begin{array}{lcr}
            l & \leftarrow & d,\ c.\\
            v & \leftarrow & l,\ b.\\
            & \leftarrow & \naf v.\\
              \end{array}\right\}
        \]
    \end{minipage}
    
    Intuitively, whenever the constraint $\leftarrow l, b$ in $C$ is violated, the rule $v \leftarrow l, b$ in $\neg C$ derives $v$, thereby preventing $\neg C$ from being coherent. Conversely, when the constraint $\leftarrow l, b$ in $C$ is satisfied, $v$ cannot be derived in $\neg C$, and hence $\neg C$ is incoherent.
\end{example}
\setcounter{prop}{1}
\begin{prop}
Given a stratified ASP program $P$, its complement $\neg P$ is coherent iff $P$ is incoherent.
\end{prop}
\setcounter{prop}{4}
\setcounter{definition}{1}
\begin{definition}[Relaxed program]\label{def:relaxed_transformation_app}
Let $P$ be a stratified program with hard constraints, $l$ be an integer, and $unsat$ be a fresh atom not appearing anywhere else. Then $relaxed(P,l)$ is defined as:
\[
    relaxed(P,l) = \left\{\begin{array}{lllr}
        H_r     & \leftarrow    & B_r.          & \forall r \in P s.t. H_r \ne \emptyset\\
        unsat   & \leftarrow    & B_r.          & \forall r \in P s.t. H_r = \emptyset\\
                & \leftarrow^w  & unsat. [1@l]  & \\
     \end{array}\right\}
    \]
\end{definition}
\setcounter{definition}{10}

Intuitively, $relaxed(\cdot, \cdot)$ is similar to $\neg C$ but the fresh atom introduced as head of hard constraints of $P$ is used to assign a penalty if $P$ is incoherent. 

The following example better clarifies how the transformation works.  
\begin{example}
    Let us consider the following stratified program $C$ and $relaxed(C,0)$
    
    \begin{minipage}{.35\textwidth}
    \[
        C=\left\{\begin{array}{lll}
        a       & \leftarrow    & d,\ c\\
                & \leftarrow    & a,\ b\\
          \end{array}\right\}
    \]    
    \end{minipage}
    \hfill
    \begin{minipage}{.65\textwidth}
    \[
        relaxed(C,0) = \left\{\begin{array}{lll}
        a       & \leftarrow    & d,\ c\\
        unsat   & \leftarrow    & a,\ b\\
                & \leftarrow^{\omega}  & unsat\ [1@0]\\  
          \end{array}\right\}
    \]        
    \end{minipage}
    In this case, $relaxed(C,0)$ is obtained from $C$ by introducing the atom $unsat$ in the head of each hard constraint in $C$. Such an atom explicitly represents constraint violations. Moreover, the weak constraint $\leftarrow^{\omega} unsat\ [1@l]$ assigns a penalty of $1$ at priority level $0$ whenever $unsat$ is true, thereby enforcing a preference to satisfy all hard constraints in $C$.
\end{example}

\setcounter{definition}{2}
\begin{definition}[Countermove program]\label{def:countermove_program_app}
Let $\Pi$ be a \TWOASPQW of the form (\ref{eq:2_aspqw_app}) and $l_{min}$ be the smallest priority level among weak constraints in $P_2$, then the \emph{countermove program} for $\Pi$ is:
    \[
    ctr(\Pi) = \left\{\begin{array}{lr}
        P_2 \cup relaxed(\neg C,l_{min}-1) & if\ \Box = \exists^{st}  \\
        P_2 \cup relaxed(C,l_{min}-1)      & if\ \Box = \forall^{st}  \\
    \end{array}\right.
    \]
\end{definition}
\setcounter{definition}{10}
\setcounter{prop}{3}
\begin{prop}\label{prop:ctr_move_app}
Let $\Pi$ be a \TWOASPQW program of the form~(\ref{eq:2_aspqw_app}), $M_1 \in OptAS(P_1)$ be a move for $\Box$. There exists $M_2 \in OptAS(ctr(\Pi)\cup \fix{P_1}{M_1})$ such that $unsat \notin M_2$ if and only if $M_2|_{\heads{P_2}}$ is a countermove to $M_1$ for $\overline{\Box}$.
\end{prop}
\setcounter{prop}{4}

The following example better clarifies the intuition behind Proposition~\ref{prop:ctr_move_app}.
\begin{example}\label{ex:twoaspw_program_example_app}
    Let $\Pi$ be the following \TWOASPQW:
    
    \begin{minipage}{.30\textwidth}
        \[
            P_1 = \left\{\begin{array}{lll}
                a  & \leftarrow & \naf na\\
                na & \leftarrow & \naf a\\
                b  & \leftarrow & \naf nb\\
                nb & \leftarrow & \naf b\\
              \end{array}\right\}
        \]
    \end{minipage}
    \begin{minipage}{.39\textwidth}
        \[
            P_2 = \left\{\begin{array}{lll}
                c  & \leftarrow             & \naf nc\\
                nc & \leftarrow             & \naf c\\
                   & \leftarrow^{\omega}    & a,\ \naf c\ [1@1]\\
                   & \leftarrow^{\omega}    & b,\ \naf nc\ [1@1]\\
              \end{array}\right\}
        \]    
    \end{minipage}
    \begin{minipage}{.27\textwidth}
        \[
            C = \left\{\begin{array}{l}
            \leftarrow b,\ c\\
            \leftarrow nb,\ nc\\
            \leftarrow b,\ a,\ nc\\
              \end{array}\right\}
        \]
    \end{minipage}
    \smallskip

    In this case $ctr(\Pi)$ is the following program:
    \[
        ctr(\Pi) = \left\{\begin{array}{lclcl}
            c\leftarrow \naf nc                     & & nc \leftarrow \naf c                    & & \\
            \leftarrow^{\omega} a,\ \naf c\ [1@1]   & & \leftarrow^{\omega} b,\ \naf nc\ [1@1]  & & \\
            v \leftarrow b,\ c                      & & v \leftarrow nb,\ nc                    & & v \leftarrow b,\ a,\ nc\\
            unsat \leftarrow \naf v                 & & \leftarrow^{\omega} unsat\ [1@0]        & & \\
          \end{array}\right\}
    \]

    Let us consider $M_1 =\{na,nb\} \in OptAS(P_1)$. 
    In this case, the optimal answer sets of $P_2\cup\fix{P_1}{M_1}$ are $\{na, nb, c\}$ and $\{na, nb, nc\}$. 
    Since $\{na, nb, nc\}$ violates the hard constraint $\leftarrow nb,nc \in C$ then $\{nc\}$ is a countermove to $M_1$ for $\overline{\Box}$.
    Similarly, let $P = ctr(\Pi)\cup\fix{P_1}{M_1}$ admits two answer sets, $M_2=\{na, nb, c, unsat\}$ and $M_2'=\{na, nb, nc, v\}$, which correspond to optimal answer sets of $P_2\cup\fix{P_1}{M_1}$.
    Here we observe that, $\mathcal{C}(P,M_2,1) = \mathcal{C}(P,M_2',1) = 0$ whereas $\mathcal{C}(P,M_2,0) > \mathcal{C}(P,M_2',0)$ as $M_2$ contains $unsat$ and $M_2'$ does not.
    Thus, $M_2'$ is the only optimal answer set of $P$. 
    Since $unsat$ is false w.r.t. $M_2'$, then $M_2'|_{\heads{P_2}} = \{nc\}$ is a countermove to $M_1$ for $\overline{\Box}$.
    
    Let us now consider $M_1=\{na,b\}$.
    In this case, the answer sets of $P_2\cup\fix{P_1}{M_1}$ are $\{na, b, c\}$ and $\{na, b, nc\}$.
    Here we can observe that $\{na, b, c\}$ violates the weak constraint $\leftarrow^{\omega} b,\naf nc \in P_2$ whereas $\{na, b, nc\}$ satisfies all weak constraints in $P_2$.
    Thus, $\{na, b, nc\}$ is the only optimal answer set of $P_2\cup\fix{P_1}{M_1}$. Since $\{na, b, nc\}$ satisfies $C$ then no countermoves to $M_1$ for $\overline{\Box}$ exists.
    Similarly, the answer sets of $P = ctr(\Pi)\cup\fix{P_1}{M_1}$ are $M_2 = \{na, b, c, v\}$ and $M_2' = \{na, b, nc,unsat\}$. 
    In this case, $M_2$ is dominated by $M_2'$ as $\mathcal{C}(P,M_2,1) > \mathcal{C}(P,M_2',1)$, and thus the only optimal answer set of $P$ is $M_2'$.
    Moreover, $unsat$ is true w.r.t $M_2'$ then no countermove to $M_1$ for $\overline{\Box}$ exists.
\end{example}

\subsection{Abstraction Refinement}\label{sec:program_transform_ref}
Let $\Pi$ be a \TWOASPQW of the form~\ref{eq:2_aspqw_app}, $M_1,M_1'$ be two moves for $\Box$, and $M_2$ be a countermove to $M_1$ for $\overline{\Box}$.

Then, then $M_2$ is not a countermove to $M_1'$ for $\overline{\Box}$ if one of the following condition holds:
\begin{enumerate}
    \item there is no $M_2' \in AS(P_2\cup\fix{P_1}{M_1'})$ such that $M_2'|_{\heads{P_2}} = M_2$\label{cond:no_ans_app};
    \item there exists $M_2' \in AS(P_2\cup\fix{P_1}{M_1'})$ such that $M_2'|_{\heads{P_2}} = M_2$ but $M_2'\notin OptAS(P_2\cup\fix{P_1}{M_1'})$ (i.e. $M_2'$ is dominated by some answer set of $P_2\cup\fix{P_1}{M_1'}$);\label{cond:no_opt_ans_app}
    \item there exists $M_2' \in OptAS(P_2\cup\fix{P_1}{M_1'})$ such that $M_2'|_{\heads{P_2}} = M_2$ but $M_2'$ satisfies $C$ and $\Box = \exists^{st}$ (resp. violates $C$ and $\Box=\forall^{st}$).\label{cond:sat_c_app}
\end{enumerate}

The following transformation will be instrumental for defining the \emph{Refinement program}.

For producing rules over a fresh signature we need to introduce a rewriting transformation.
To this end, we introduce a predicate substitution function that maps each predicate $p$ with a fresh ones of the form $p^{\alpha}_\beta$, where $\alpha$ and $\beta$ are strings, to obtain rules over a fresh signature for each discovered countermove.
More in detail, let $L$ be a set of literals and $\epsilon$ be an ASP expression (i.e. rule, program, etc.), then $\sigma^{\alpha}_\beta(L,\epsilon)$ denotes the expression obtained from $\epsilon$ by mapping each positive (resp. negative) literal $p(\vec{t}) \in L$ (resp. $\naf p(\vec{t}) \in L$) with $p^{\alpha}_\beta(\vec{t})$ (resp. $\naf p^{\alpha}_\beta(\vec{t})$).

\setcounter{definition}{3}
\begin{definition}[Check Answer Set]\label{def:same_transformation_app}
    Let $P$ be an ASP program and $M$ be an interpretation, then $checkAS(P,M)$ is:
    \[
    checkAS(P, M) = \left\{\begin{array}{lr}
    \sigma_M^-(\heads{P},\{a \leftarrow \vert a \in M\})&\\
    \sigma_M^+(\heads{P}, \sigma_M^-(\overline{\heads{P}},r)) & \forall r \in P, H_r\neq \emptyset\\
    fail_M \leftarrow \sigma_M^+(\heads{P}, \sigma_M^-(\overline{\heads{P}}, B_r)) & \forall r \in P, H_r = \emptyset\\
    fail_M \leftarrow p(\vec{t})_M^+,\naf p(\vec{t})_M^- & \forall\ p(\vec{t}) \in \heads{P} \\
    fail_M \leftarrow p(\vec{t})_M^-,\naf p(\vec{t})_M^+ & \forall\ p(\vec{t}) \in \heads{P}\\
    as_M \leftarrow \naf fail_M\ & \
    \end{array}\right.
    \]
\end{definition}
\setcounter{definition}{10}

The following example better clarifies how $CheckAS(\cdot,\cdot)$ works.
\begin{example}
    Consider $P_2$ from Example~\ref{ex:twoaspw_program_example_app} and $M = \{c\}$ then $checkAS(P_2,M)$ is of the following form:
    \[
    checkAS(P_2, M) = \left\{\begin{array}{lcl}
    c^{+}_{M} \leftarrow \naf nc^{-}_M   & \quad & nc^{+}_{M} \leftarrow \naf c^{-}_M\\
    c^{-}_{M} \leftarrow                 &       & \\
    fail_M \leftarrow c_M^+,\naf c_M^-   & \quad & fail_M \leftarrow nc_M^+,\naf nc_M^- \\
    fail_M \leftarrow c_M^-,\naf c_M^+   & \quad & fail_M \leftarrow nc_M^-,\naf nc_M^+ \\
    as_M \leftarrow \naf fail_M\\
    \end{array}\right\}
    \]
    In this case, the program $P_2$ is made of rules $c \leftarrow \naf nc$ and $nc \leftarrow \naf c$.
    These two rules are rewritten as follows: ($i$) negative literals over atoms in $\heads{P_2}$ (i.e. $\overline{\heads{P_2}} = \{\naf c,\naf nc\}$) are replaced, respectively, by $\naf c_M^-$ and $\naf nc_M^-$, which form the negative signature; and ($ii$) positive literals over atoms in $\heads{P_2}$ (i.e. $\heads{P_2} = \{c, nc\}$) are replaced, respectively, by $c_M^+$ and $nc_M^+$, which form positive signature.
    The interpretation $M = \{c\}$ is then encoded as facts over the negative signature, which allows the previous rules to be simplified thus yielding the the GL-reduct of $P_2$ w.r.t. $M$ (i.e. $P_2^M$).
    Next, the rules having $fail_M$ in the head check whether the truth values of atoms over the positive signature (i.e. a model of the reduct) match the truth value of those over the negative signature (i.e. the interpretation $M$).
    Finally, the atom $as_M$ is derived if and only if the truth values of all atoms over the two signatures coincide, that is, $fail_M$ cannot be derived as true.
\end{example}

\setcounter{definition}{5}
\begin{definition}[Dominated program]\label{dominated_transformation_app}
    Let $P_1$ and $P_2$ be two ASP programs, then $checkDom(P_1,P_2)$ is the following program: 
    \[
    \begin{array}{lclr}
    diff_{P_1,P_2} & \leftarrow & cl_{P_1}(C1,L), cl_{P_2}(C2,L), C1 \neq C2. & \\
    hasHigher_{P_1,P_2}(L) & \leftarrow & diff_{P_1,P_2}(L), diff_{P_1,P_2}(L1), L<L1.& \\
    highest_{P_1,P_2}(L) & \leftarrow & diff_{P_1,P_2}(L), \naf hasHigher_{P_1,P_2}(L).& \\
    dom_{P_1,P_2} & \leftarrow & highest_{P_1,P_2}(L), cl_{P_1}(C1,L), cl_{P_2}(C2,L), C2 < C1 & \\
    \end{array}
    \]
    where $cl_{P_1}$ and $cl_{P_2}$ are, respectively, the predicates introduced by $cost(P_1)$ and $cost(P_2)$ (Definition~\ref{cost_transformation_app}); whereas $diff$, $hasHigher$, $highest$, and $dominated$ are fresh predicates.
\end{definition}

\begin{definition}[Controlled program]\label{controlled_program_transformation_app}
    Let $P$ be an ASP program and $l$ be a literal not appearing in $P$, then 
    $or(P,l)$ is the program of the form $\{H_r \leftarrow B_r, l \mid r \in P\}$
\end{definition}

\begin{definition}[Refinement Program]\label{def:refinement_program_app}
Let $\Pi$ be a \TWOASPQW program of the form~(\ref{eq:2_aspqw_app}), $M$ be a move for $\Box$ and $CE$ be a countermove to $M$ for $\overline{\Box}$. The \textit{refinement program} is defined as follows:
     \[
    ref(\Pi, CE) = \left\{\begin{array}{lr}
        checkAS(P_2,CE) & \\
        \leftarrow^w as_{CE}\ [1@l_{min}-1] & \\
        or(cost(P_2^{CE}),as_{CE}) & \\
        or(clone(\mathcal{R}(P_2)),as_{CE}) & \\
        or(cost(clone(P_2)),as_{CE}) & \\
        or(checkDom(P_2^{CE}, clone(P_2)),as_{CE}) & \\
        \leftarrow^w as_{CE}, \naf dom_{P_2^{CE},clone(P_2)}\ [1@l_{min}-2] & \\
        or(\sigma_{CE}^-(lits(P_2)\cup lits(C'), C'),as_{CE}) & \\
    \end{array}\right\}
    \]
where $P_2^{CE} = \sigma_{CE}^-(lits(P_2),P_2)$, $as_{CE}$ is the predicate introduced by $checkAS(P_2, CE)$, $C' = relaxed(C,l_{min}-3)$ if $\Box=\exists^{st}$; otherwise $C' = relaxed(\neg C,l_{min}-3)$, and $l_{min}$ is the smallest priority level among weak constraints of $P_1$.
\end{definition}
\setcounter{definition}{10}

To explain the idea behind $ref(\Pi,CE)$ we use the following example.
\begin{example}
Let $\Pi$ be the \TWOASPQW program from Example~\ref{ex:twoaspw_program_example_app}, we know that $M_1 = \{na,nb\} \in OptAS(P_1)$ is a move for $\Box$ and $M_2 = \{nc\}$ is a countermove to $M_1$ for $\overline{\Box}$. We therefore aim to refine the abstraction $P_1$ as to avoid all those optimal answer sets of $P_1$ that admit $M_2$ as countermove. 
To this end, we use the rules obtained by $Ref(\Pi, M_2)$ which are reported below.

\smallskip\noindent
\begin{minipage}[t]{.38\textwidth}
\noindent \% $checkAS(P_2,M_2)$
\begin{flalign*}
&
\begin{array}{l}
        c_{M_2}^+   \leftarrow \naf nc_{M_2}^-\\
        nc_{M_2}^+  \leftarrow          \naf c_{M_2}^-\\
        nc_{M_2}^-  \leftarrow          \\
        fail_{M_2}  \leftarrow          c_{M_2}^+ \naf c_{M_2}^-\\
        fail_{M_2}  \leftarrow          nc_{M_2}^+ \naf nc_{M_2}^-\\
        as_{M_2}    \leftarrow          \naf fail_{M_2}\\
                    \leftarrow^{\omega} as_{M_2}\ [1@-1]\\
\end{array}
&
\end{flalign*}    
\end{minipage}
\begin{minipage}[t]{.49\textwidth}
   \noindent \% $or(cost(P_2^{M_2}),as_{M_2})$
    \begin{flalign*}
    &
    \begin{array}{l}
            v_{P_2^{M_2}}(1,1)  \leftarrow a,\ \naf c_{M_2}^-, as_{M_2}\\
            v_{P_2^{M_2}}(1,1)  \leftarrow b,\ \naf nc_{M_2}^-, as_{M_2}\\
            cl_{P_2^{M_2}}(C,1) \leftarrow \#sum\{1,1:v_{P_2^{M_2}}(1,1)\}=C, as_{M_2}\\
    \end{array}
    &
    \end{flalign*} 
\end{minipage}

\smallskip\noindent
\begin{minipage}[t]{.38\textwidth}
\noindent \% $or(clone(\mathcal{R}(P_2)),as_{M_2})$
\begin{flalign*}
&
\begin{array}{l}
        c_{clone}   \leftarrow \naf nc_{clone}, as_{M_2}\\
        nc_{clone}  \leftarrow \naf c_{clone}, as_{M_2}\\
\end{array}
&
\end{flalign*}    
\end{minipage}
\begin{minipage}[t]{.49\textwidth}
    \noindent \% $or(cost(clone(P_2)),as_{M_2})$
    \begin{flalign*}
    &
    \begin{array}{l}
            v_{clone(P_2)}(1,1) \leftarrow a,\ \naf c_{clone}, as_{M_2}\\
            v_{clone(P_2)}(1,1) \leftarrow b,\ \naf nc_{clone}, as_{M_2}\\
            cl_{clone(P_2)}(C,1)\leftarrow \#sum\{1,1:v_{clone(P_2)}(1,1)\}=C, as_{M_2}\\
    \end{array}
    &
    \end{flalign*}
\end{minipage}

\smallskip\noindent
\begin{minipage}{\textwidth}
    \noindent \% $or(checkDom(P_2^{M_2}, clone(P_2)),as_{M_2}))$
\begin{flalign*}
&
\begin{array}{l}
        diff_{P_2^{M_2}, clone(P_2)} \leftarrow cl_{P_2^{M_2}}(C1,L), cl_{clone(P_2)}(C2,L), C1 \ne C2, as_{M_2}\\
        hasHigher_{P_2^{M_2},clone(P_2)}(L) \leftarrow diff_{P_2^{M_2}, clone(P_2)}(L), diff_{P_2^{M_2}, clone(P_2)}(L1), L<L1, as_{M_2}\\
        highest_{P_2^{M_2}, clone(P_2)}(L) \leftarrow  diff_{P_2^{M_2}, clone(P_2)}(L), \naf hasHigher_{P_2^{M_2}, clone(P_2)}(L), as_{M_2}\\
        dom_{P_2^{M_2}, clone(P_2)} \leftarrow highest_{P_2^{M_2}, clone(P_2)}(L), cl_{P_2^{M_2}, clone(P_2)}(C1,L), as_{M_2}\\
                                    \leftarrow^{\omega} \naf as_{M_2}, \naf dom_{P_2^{M_2},clone(P_2)}\ [1@-2]\\
\end{array}
&
\end{flalign*}
\end{minipage}

\smallskip\noindent
\begin{minipage}{.6\textwidth}
\noindent \% $or(\sigma_{M_2}^-(lits(P_2)\cup lits(C), relaxed(C,-3)),as_{M_2})$
\begin{flalign*}
&
    \begin{array}{l}
            unsat_{M_2}^{-} \leftarrow b, c_{M2}^{-}, as_{M_2}.\\
            unsat_{M_2}^{-} \leftarrow nb, nc_{M_2}^{-}, as_{M_2}.\\
            unsat_{M_2}^{-} \leftarrow b, a, nc_{M_2}^{-}, as_{M_2}.\\
                            \leftarrow^{\omega} unsat_{M_2}^{-}.\ [1@-3] \\
    \end{array}
&
\end{flalign*}
\end{minipage}
\smallskip

As it can be observed, the refinement rules are made of several blocks which encode the the three conditions for detecting if $M_2$ is again a countermove. 
More precisely, the first block contains $checkAS(P_2,M_2)$ and thus simulates the reduct of $P_2$ w.r.t. $M_2$. 
Here, we can observe that $M_2 = \{nc\}$ is encoded as the fact $nc_{M_2}^-\leftarrow$ over the negative signature. Since $nc_{M_2}^-$ is true, the rule $c_{M_2}^+\leftarrow \naf nc_{M_2}^-$ has a false body, whereas the rule $nc_{M_2}^+\leftarrow \naf c_{M_2}^-$ can be simplified as $c_{M_2}^-$ does not appear in the head of any rule and is therefore false. As a result, during the evaluation we will obtain the reduct of $P_2$ w.r.t. $M_2$.
The remaining rules derive the atom $as_{M_2}$ if and only if the negative signature matches the positive one (i.e. $fail_{M_2}$ is false). 
If $as_{M_2}$ is derived as false, then condition~\ref{cond:no_ans_app} is satisfied.
Consequently, the weak constraint $\leftarrow^{\omega} as_{M_2}$ enforces a preference for moves that satisfy condition~\ref{cond:no_ans_app}.
Moreover, when $as_{M_2}$ is derived as false, since $as_{M_2}$ appears in the body all the remaining rules, then these rules are trivially satisfied and no further checks are required. Conversely, if $as_{M_2}$ is derived as true, all the remaining rules are activated in order to 
verify the other conditions.

To check condition~\ref{cond:no_opt_ans_app}, we instead use the blocks $\%or(cost(P_2^{M_2}),as_{M_2})$ - \%$or(checkDom(P_2^{M_2}$ and $clone(P_2)),as_{M_2})$.
More precisely, the rules in $cost(P_2^{M_2})$ compute the cost of $M_2$ w.r.t. weak constraints in $P_2$. Note that, $P_2^{M_2}$ is obtained as $\sigma_{M_2}^-(lits(P_2),P_2)$ since $M_2$ has been encoded as facts over the negative signature (i.e. $nc_{M_2}^-$). Thus, the negative signature for atoms in $P_2$ is used to compute the cost of $M_2$.
To verify that $M_2$ is optimal, it must be compared with all answer sets of $P_2$. This one-to-one comparison, is achieved by adding a clone of the rules from program $P_2$ (i.e. the block $clone(\mathcal{R}(P_2))$). Note that, the weak constraints in $P_2$ are not cloned in order to preserve all the answer sets of $P_2$. The rules in $cost(clone(P_2))$ are then used to compute the cost of an answer set $M_{clone}$ of $clone(P_2)$, in the same spirit as $cost(P_2^{M_2})$. As a result, atoms of the form $cl_{P_2^{M_2}}(\cdot,\cdot)$ and $cl_{clone(P_2)}(\cdot,\cdot)$ encode the cost of $M_2 \in AS(P_2)$ and $M_{clone} \in AS(clone(P_2))$, respectively.
Finally, rules in the block $checkDom(P_2^{M_2}, clone(P_2))$ verify whether $M_2$ is dominated by $M_{clone}$. More precisely, the atom $dom_{P_2^{M_2},clone(P_2)}$ is derived as true if and only if $M_2$ is dominated by $M_{clone}$, and thus condition~\ref{cond:no_opt_ans_app} is satisfied. Consequently, the weak constraint $\leftarrow^{\omega} as_{M_2},\naf dom_{P_2^{M_2},clone(P_2)}\ [1@l_{min}-2]$ enforces a preference for moves that satisfy condition~\ref{cond:no_opt_ans_app}.

The last condition, namely condition~\ref{cond:sat_c_app}, is encoded by the block $or(\sigma_{M_2}^-(lits(P_2)\cup lits(C), relaxed(C,-3)), as_{M_2})$.
The rules contained in this block check whether $C$ is satisfied or not. 
In this example, condition~\ref{cond:sat_c_app} is satisfied if and only if $C$ is coherent. Accordingly, $relaxed(C,-3)$ is used to capture the incoherence of $C$ by means of the atom $unsat$. 
Note that the program $C$ may contain atoms appearing in $P_2$.
Since $M_2$ has been encoded as facts over the negative signature (i.e., $nc_{M_2}^-)$,the rules from $relaxed(C,-3)$ must be mapped onto this signature using $\sigma_{M_2}^-(\cdot)$.
As a result, the atom $unsat$ is mapped to $unsat_{M_2}^-$, which is derived as false if and only if $M_2$ satisfies $C$, and thus condition~\ref{cond:sat_c_app} holds. Consequently, the weak constraint $\leftarrow^{\omega} unsat_{M_2}^-\ [1@-3]$ enforces a preference for moves that satisfy condition~\ref{cond:sat_c_app}.

Finally, the next move for $\Box$ can be computed from an optimal answer set $M_1'$ of $P_1\cup ref(\Pi,M_2)$. 
Note that, if all the weak constraints introduced by $ref(\Pi,M_2)$ are violated w.r.t. $M_1'$, then conditions~\ref{cond:no_ans_app},~\ref{cond:no_opt_ans_app}, and~\ref{cond:sat_c_app} are violated, and so, $M_2$ is a countermove also to $M_1'$. Hence, $M_1'$ can be considered as the next move for $\Box$ if and only if at least one weak constraint in $ref(\Pi,M_2)$ is satisfied w.r.t. $M_1'$.
Conversely, if $M_1'$ violates all the weak constraints of $ref(\Pi,M_2)$, then all the optimal answer sets of $P_1\cup ref(\Pi,M_2)$ violate the weak constraints in $ref(\Pi,M_2)$. In this case, all possible moves for $\Box$ have $M_2$ as countermove and thus no further move exists for $\Box$.
\end{example}

\section{Solving \TWOASPQW Programs}\label{sec:extended_cegar}
In this section we give a more detailed description of the CEGAR-based algorithm for solving \TWOASPQW.

\begin{algorithm}
\footnotesize
    \caption{CEGAR for \TWOASPQW}
    \textbf{Input}: A \TWOASPQW of the form $\Box_1 P_1 \overline{\Box_1} P_2:C$\\
    \textbf{Output}:$M_1$ if a winning move exists for $\Box_1$, $NULL$ otherwise
    \begin{algorithmic}[1]\label{alg_2aspq_app}
        \STATE $CM_S = \emptyset$\label{alg:initialize_countermoves_app}
        \STATE $M_1 = solveOpt(P_1)$\label{alg:first_candidate_app}
        \WHILE{$M_1 \ne \bot$}\label{alg:open_while_app}
        \STATE $M_2 = solveOpt(ctr(\Pi) \cup \fix{P_1}{M_1})$\label{alg:solve_ce_program_app}
        \IF{$M_2 = \bot \lor \mathcal{C}(ctr(\Pi) \cup \fix{P_1}{M_1}, M_2, 0)=1$} \STATE $return\ M_1$\label{alg:no_counterexample_app}
        \ELSE
        \STATE $CE = M_2 |_{\heads{P_2}}$\label{alg:ce_found_app}
        \STATE $CM_s = CM_s \cup CE$\label{alg:add_ce_to_countermoves_app}
        \ENDIF
        \STATE $A = P_1 \bigcup_{CE \in CM_s} ref(\Pi, CE)$\label{alg:construct_refinement_app}
        \STATE $N = solveOpt(A)$\label{alg:find_new_candidate_app}
        \IF{$(\mathcal{C}(A, M_1,l_{min}-1) = 1 \land \mathcal{C}(A, M_1,l_{min}-2) = 1 \land \mathcal{C}(A, M_1,l_{min}-3) = 1 )$}
        \STATE $return\ NULL$\label{alg_no_new_canidate_weak_app}
        \ENDIF\\
        $M_1 = N|_{\heads{P_1}}$ 
        \ENDWHILE
        \STATE $return\ NULL$\label{alg:no_new_candidate_unsat_app}
    \end{algorithmic}
\end{algorithm}

Based on the definitions of \emph{countermove program} (i.e., Definition~\ref{def:countermove_program_app}) and \emph{refinement program} (i.e., Definition~\ref{def:refinement_program_app}), we can construct a procedure based on CEGAR for deciding the coherence of \TWOASPQW programs (see Algorithm~\ref{alg_2aspq_app}).
In particular, such a procedure relies on an off-the-shelf ASP solver, which is used to compute optimal answer sets of both the (refined) abstraction and the countermove program.
We denote by $solveOpt(P)$ the call to an ASP solver that returns an optimal answer set of an ASP program $P$, if one exists.  
Algorithm~\ref{alg_2aspq_app} takes as input a \TWOASPQW of the form~\ref{eq:2_aspqw_app} and returns a winning move for $\Box$ if any, otherwise returns $NULL$.
The algorithm starts by initializing the set of countermoves to the empty set (line~\ref{alg:initialize_countermoves_app}) and by computing an optimal answer set of the abstraction $P_1$ (line~\ref{alg:first_candidate_app}).
If $P_1$ is incoherent, the 
algorithm terminates returning $NULL$, since no winning move for $\Box$ exists (line~\ref{alg:no_new_candidate_unsat_app}).
Otherwise, $P_1$ admits an optimal answer set $M_1$, which is a move for $\Box$. Algorithm~\ref{alg_2aspq_app} then enters the CEGAR loop, which alternates countermove search and abstraction refinement steps (line~\ref{alg:open_while_app}). 

Let $M_1$ be a move for $\Box$. At each iteration of the loop, a countermove to $M_1$ is searched by solving $ctr(\Pi) \cup \fix{P_1}{M_1}$ (line~\ref{alg:solve_ce_program_app}). 
If $ctr(\Pi) \cup \fix{P_1}{M_1}$ is incoherent, then $P_2\cup\fix{P_1}{M_1}$ is incoherent and hence no candidate countermove to $M_1$ for $\overline{\Box}$ exists.
In this case, Algorithm~\ref{alg_2aspq_app} returns $M_1$ as winning move.
Otherwise, $crt(\Pi)\cup \fix{P_1}{M_1}$ is coherent, and $solveOpt(crt(\Pi)\cup \fix{P_1}{M_1})$ returns an optimal answer set $M_2$.
If $unsat\in M_2$, then there exists no $M_2' \in OptAS(ctr(\Pi)\cup\fix{P_1}{M_1})$ such that $unsat \notin M_2'$, otherwise $M_2$ would be dominated by $M_2'$.
Hence, from Proposition~\ref{prop:ctr_move_app}, none of the candidate countermoves is an actual countermove to $M_1$ for $\overline{\Box}$, and so Algorithm~\ref{alg_2aspq_app} returns $M_1$ as winning move.
If, instead, $M_2$ does not contain the atom $unsat$, then from Proposition~\ref{prop:ctr_move_app}, $M_2|_{\heads{P_2}}$ is a countermove to $M_1$ for $\overline{\Box}$.
Accordingly, Algorithm~\ref{alg_2aspq_app} extracts the countermove $CE$ by projecting $M_2$ onto the atoms appearing in the heads of the rules in $P_2$ (line~\ref{alg:ce_found_app}), and adds $CE$ to the set of countermoves $CMs$ (line~\ref{alg:add_ce_to_countermoves_app}). 
The algorithm then computes the refined abstraction $A$ by refining the abstraction $P_1$ according to the the countermoves in $CMs$ (line~\ref{alg:construct_refinement_app}). 
The refined abstraction is subsequently solved to compute the next move for $\Box$ (line~\ref{alg:find_new_candidate_app}). 
Let $N$ be an answer set of the refined abstraction $A$, and $M_1 = N|_{\heads{P_1}}$ be a move for $\Box$. If there exists some $CE \in CMs$ such that $N$ violates all the weak constraints in $ref(\Pi,CE)$, then $CE$ is again a valid countermove to $M_1$. In this case, the cost of $N$ at priority levels $l_{min}-1$, $l_{min}-2$, and $l_{min}-3$ is equal to $1$, which is the maximum possible cost at these levels.
Accordingly, if the cost of an optimal answer set of the refined abstraction $A$ at levels $l_{min}-1$, $l_{min}-2$, and $l_{min}-3$ is equal to $1$, then every possible move admits at least one $CE \in CMs$ as a countermove. Hence, no winning move for $\Box$ exists, and Algorithm~\ref{alg_2aspq_app} returns $NULL$ (line~\ref{alg_no_new_canidate_weak_app}).
In all other cases, $M_1$ does not admit any countermove in $CMs$, and the loop continues by taking $M_1$ as the next move for $\Box$.

\end{document}